%% file: sample_paper.tex
\documentclass[twoside]{article}

\usepackage[preprint]{aistats2027}

\usepackage[utf8]{inputenc} % allow utf-8 input
\usepackage{lmodern}
\usepackage[T1]{fontenc}    % use 8-bit T1 fonts
\usepackage{tikz}
\usepackage{url}            % simple URL typesetting
\usepackage{booktabs}       % professional-quality tables
\usepackage{amsfonts}       % blackboard math symbols
\usepackage{microtype}      % microtypography
\usepackage{xcolor}         % colors
\pdfmapfile{+dstroke.map}
\usepackage{dsfont}         % for \mathds{1} indicator function
\usepackage{verbatim}      % for comment environment
\usepackage{amsmath}
\usepackage{amssymb}
\usepackage{amsthm}

\usepackage{xcolor}         % colors
\usepackage{placeins}      % for \FloatBarrier
\usepackage[hidelinks]{hyperref}

\newtheorem{theorem}{Theorem}
\newtheorem{proposition}{Proposition}
\newtheorem{corollary}{Corollary}
\newtheorem{lemma}{Lemma}

\theoremstyle{definition}
\newtheorem{assumption}{Assumption}

\theoremstyle{remark}

\numberwithin{equation}{section}

\usepackage[round]{natbib}

\newcommand{\Sem}{\mathrm{Sem}}

\DeclareMathOperator{\TV}{TV}
\DeclareMathOperator{\KL}{KL}
\DeclareMathOperator{\Var}{Var}

\begin{document}

% If your paper is accepted and the title of your paper is very long,
% the style will print as headings an error message. Use the following
% command to supply a shorter title of your paper so that it can be
% used as headings.
%
%\runningtitle{I use this title instead because the last one was very long}

% If your paper is accepted and the number of authors is large, the
% style will print as headings an error message. Use the following
% command to supply a shorter version of the author names so that
% they can be used as headings (for example, use only the surnames)
%
%\runningauthor{Surname 1, Surname 2, Surname 3, ...., Surname n}

\twocolumn[

\aistatstitle{On the Indistinguishability of Human v/s AI Generated Text}

\aistatsauthor{ Jaee Ponde \And Aritra Das \And  Mihir More \And Debayan Gupta }

\aistatsaddress{ Truth Audit Labs } ]

\begin{abstract}
The rapid improvement of LLMs has made distinguishing AI-generated text from human writing a pressing problem. This challenge is further amplified by paraphrasing tools designed to make machine-generated text appear more ``human''. We study how access to human writing samples can be used to strategically paraphrase machine-generated responses toward the human distribution. Under a multi-sample setting with human and machine responses to the same prompts, we show that repeated paraphrasing moves the machine distribution toward a human target under mixing and stability conditions. We derive an explicit convergence rate, examine how semantic-preservation and quality constraints affect this rate, and extend the analysis to finite samples to characterize how the required human data and paraphrasing rounds scale with the desired error. We extend the analysis beyond the shared-parameter setting by learning a human-preserving rewriting kernel.
\end{abstract}

\section{Introduction}

Large language models can now produce fluent text that is increasingly difficult to distinguish from human writing. This has led to growing use of AI-text detectors, but reliable detection ultimately relies on statistical differences between human and machine distributions. Prior work shows that this distinction can become weak at the single-text level and that paraphrasing can substantially reduce detector performance while preserving meaning \citep{sadasivan2023reliably,krishna2023paraphrasing,cheng2025adversarial}. At the same time, small distributional differences can accumulate when many independent samples are available making detection promising \citep{chakraborty2024possibilities}.

This motivates an adversarial version of the multi-sample problem. Consider an adversary which has both machine-generated responses and examples of how a particular human writes: Can the machine responses be repeatedly rewritten to be distributionally closer to that human's writing while preserving the original content and quality? If so, how quickly does this happen, and how much human data is required?

We model repeated paraphrasing as a random walk on a graph. The nodes represent admissible texts that express the same content, and the edges represent allowed rewrites; transition probabilities are determined by a scoring model and an acceptance rule. We study how this process reshapes the distribution of machine writing across rewrites and how human examples can guide it towards the human distribution. Our analysis further examines how the rate of convergence depends on the graph and the rewriting probabilities.

We analyze two complementary settings. First, we assume that human writing across different prompts can be described by a shared parameter, which can be estimated from human examples and used to guide rewriting. We then remove this assumption and learn a rewriting rule directly from human examples. Our main contributions are:
\begin{itemize}
    \item We prove geometric convergence of repeated rewriting and characterize how rewrite acceptance constraints affect the convergence rate (Theorem~\ref{thm:mixing}, Proposition~\ref{prop:acceptance-comparison}).
    \item We show that the human target parameter can be estimated from finitely many human examples at a square-root rate (Theorem~\ref{thm:beta-estimation}, Lemma~\ref{lem:parameter-sensitivity}).
    \item We combine rewriting and calibration errors into finite-sample guarantees for approaching the human distribution and the effect of sampling only finitely many rewritten outputs (Theorems~\ref{thm:finite-sample}--\ref{thm:empirical-output}).
    \item We extend the analysis beyond the shared-parameter model by learning a ``human-preserving'' rewriting kernel from data (Theorems~\ref{thm:approx-preservation}--\ref{thm:direct-finite-sample}).
\end{itemize}

\input{related}
\input{sections/setup}
\input{sections/convergence_rate}
\input{sections/drift}
\input{sections/sample_complexity}
\input{sections/convergence_without_tempered}

\section*{AI Use Statement}
In this work, we used generative AI tools only for manuscript editing,
organization, notation consistency checks, and LaTeX formatting. We take responsibility for the final content of this work, including
text, claims, or artifacts produced with the aid of generative AI.

\bibliography{references}
%%%%%%%%%%%%%%%%%%%%%%%%%%%%%%%%%%%%%%%%%%%%%%%%%%%%%%%%%%%%
\section*{Checklist}

\begin{comment}
% %%% BEGIN INSTRUCTIONS %%%
The checklist follows the references. For each question, choose your answer from the three possible options: Yes, No, Not Applicable.  You are encouraged to include a justification to your answer, either by referencing the appropriate section of your paper or providing a brief inline description (1-2 sentences).
Please do not modify the questions.  Note that the Checklist section does not count towards the page limit. Not including the checklist in the first submission won't result in desk rejection, although in such case we will ask you to upload it during the author response period and include it in camera ready (if accepted).

\textbf{In your paper, please delete this instructions block and only keep the Checklist section heading above along with the questions/answers below.}
% %%% END INSTRUCTIONS %%%
\end{comment}

\begin{enumerate}

  \item For all models and algorithms presented, check if you include:
  \begin{enumerate}
    \item A clear description of the mathematical setting, assumptions, algorithm, and/or model. [Yes]
    \item An analysis of the properties and complexity (time, space, sample size) of any algorithm. [Not Applicable]
    \item (Optional) Anonymized source code, with specification of all dependencies, including external libraries. [Yes]
  \end{enumerate}

  \item For any theoretical claim, check if you include:
  \begin{enumerate}
    \item Statements of the full set of assumptions of all theoretical results. [Yes]
    \item Complete proofs of all theoretical results. [Yes]
    \item Clear explanations of any assumptions. [Yes]
  \end{enumerate}

  \item For all figures and tables that present empirical results, check if you include:
  \begin{enumerate}
    \item The code, data, and instructions needed to reproduce the main experimental results (either in the supplemental material or as a URL). [Not Applicable]
    \item All the training details (e.g., data splits, hyperparameters, how they were chosen). [Not Applicable]
    \item A clear definition of the specific measure or statistics and error bars (e.g., with respect to the random seed after running experiments multiple times). [Not Applicable]
    \item A description of the computing infrastructure used. (e.g., type of GPUs, internal cluster, or cloud provider). [Not Applicable]
  \end{enumerate}

  \item If you are using existing assets (e.g., code, data, models) or curating/releasing new assets, check if you include:
  \begin{enumerate}
    \item Citations of the creator If your work uses existing assets. [Not Applicable]
    \item The license information of the assets, if applicable. [Not Applicable]
    \item New assets either in the supplemental material or as a URL, if applicable. [Not Applicable]
    \item Information about consent from data providers/curators. [Not Applicable]
    \item Discussion of sensible content if applicable, e.g., personally identifiable information or offensive content. [Not Applicable]
  \end{enumerate}

  \item If you used crowdsourcing or conducted research with human subjects, check if you include:
  \begin{enumerate}
    \item The full text of instructions given to participants and screenshots. [Not Applicable]
    \item Descriptions of potential participant risks, with links to Institutional Review Board (IRB) approvals if applicable. [Not Applicable]
    \item The estimated hourly wage paid to participants and the total amount spent on participant compensation. [Not Applicable]
  \end{enumerate}

\end{enumerate}

\clearpage
\appendix
\thispagestyle{empty}

% Supplementary material: To improve readability, you must use a single-column format for the supplementary material.
\onecolumn
\aistatstitle{Instructions for Paper Submissions to AISTATS 2027: \\
Supplementary Materials}

\section{Proofs for Section~\ref{sec:setup}}\label{app:setup-proofs}

\subsection{Proof of Proposition~\ref{prop:stationarity}}

For each $i$, $p_{c,\beta,i}(\cdot\mid y_{-i})$ is a probability distribution, so $K_{c,\beta}$ is a Markov kernel. Since $0\le O_c\le1$, the definition~\eqref{eq:3.8} gives nonnegative transition probabilities whose rows sum to one. Hence $T_{c,\beta}[O]$ is also a Markov kernel.

Consider distinct states $y,z$ that differ only at sentence $i$. Since $y_{-i}=z_{-i}$,
\[
    K_{c,\beta}(y,z)
    =
    \frac1{n_c}
    \frac{e^{\beta L_c(z)}}
    {W_{c,\beta,i}(y_{-i})}.
\]
Using~\eqref{eq:target-family},
\[
\begin{aligned}
    \pi_{c,\beta}(y)T_{c,\beta}[O](y,z)
    &=
    \frac{e^{\beta L_c(y)}}{Z_c(\beta)}
    O_c(y,z)
    \frac{e^{\beta L_c(z)}}
    {n_cW_{c,\beta,i}(y_{-i})}.
\end{aligned}
\]
The right-hand side is unchanged when $y$ and $z$ are exchanged, because $O_c$ is symmetric and $y_{-i}=z_{-i}$. Therefore
\[
    \pi_{c,\beta}(y)T_{c,\beta}[O](y,z)
    =
    \pi_{c,\beta}(z)T_{c,\beta}[O](z,y).
\]
The identity is trivial on the diagonal, so detailed balance holds for all $y,z$. Summing over $y$ shows that $\pi_{c,\beta}$ is stationary.

If the accepted graph is connected, the finite-state chain is
irreducible and therefore has a unique stationary distribution
\citep{levin2017markov}.

Finally, every proposed move replaces one sentence by another element of the corresponding set $S_{c,i}$. Hence the process never leaves $\mathcal A_c$. By the admissibility condition~\eqref{eq:2.2}, every state reached by the chain satisfies the required quality and semantic-preservation constraints.
\qed

\section{Technical Results for Mixing and Convergence}\label{app:mixing}

Throughout this appendix, fix a content class and suppress its index $c$ unless needed. Write $\pi_\beta=\pi_{c,\beta}$, $T_\beta[O]=T_{c,\beta}[O]$, and $K_\beta=K_{c,\beta}$.

\subsection{Dirichlet form and spectral gap}\label{app:mixing-basic}

Define the stationary flow
\[
    \omega_\beta[O](y,z)
    =
    \pi_\beta(y)T_\beta[O](y,z)
\]
and, for functions $f,g$ on the state space, the Dirichlet form
\[
    \mathcal E_\beta[O](f,g)
    =
    \frac12
    \sum_{y,z}
    \omega_\beta[O](y,z)
    \bigl(f(z)-f(y)\bigr)
    \bigl(g(z)-g(y)\bigr).
\]
We use the inner product
\[
    \langle f,g\rangle_{\pi_\beta}
    =
    \sum_y\pi_\beta(y)f(y)g(y).
\]

For each sentence position $i$, let
\[
    E_{\beta,i}f
    =
    \mathbb E_{\pi_\beta}
    \bigl[f(Y)\mid Y_{-i}\bigr].
\]

\begin{lemma}[Positive semidefiniteness]\label{lem:psd}
The Gibbs proposal kernel satisfies
\[
    K_\beta
    =
    \frac1n\sum_{i=1}^n E_{\beta,i},
\]
and
\begin{equation}\label{eq:app-dirichlet-gibbs}
    \mathcal E_\beta[1](f,f)
    =
    \frac1n
    \sum_{i=1}^n
    \mathbb E_{\pi_\beta}
    \left[
        \Var_{\pi_\beta}
        \bigl(f(Y)\mid Y_{-i}\bigr)
    \right].
\end{equation}
Both $K_\beta$ and $T_\beta[O]$ are self-adjoint and positive semidefinite on $L^2(\pi_\beta)$. If the accepted graph is connected, then
\[
    \mathcal E_\beta[O](f,f)=0
\]
if and only if $f$ is constant.
\end{lemma}

\begin{proof}
Updating coordinate $i$ replaces $f$ by its conditional expectation given $Y_{-i}$, so $E_{\beta,i}$ is an orthogonal projection in $L^2(\pi_\beta)$. Hence
\[
    \langle f,K_\beta f\rangle_{\pi_\beta}
    =
    \frac1n\sum_i
    \|E_{\beta,i}f\|_{L^2(\pi_\beta)}^2
    \ge0.
\]
The conditional variance identity gives~\eqref{eq:app-dirichlet-gibbs}.

Detailed balance from Proposition~\ref{prop:stationarity} makes both kernels self-adjoint, and
\[
    \mathcal E_\beta[O](f,f)
    =
    \langle f,(I-T_\beta[O])f\rangle_{\pi_\beta}.
\]
Since $0\le O\le1$,
\[
    0
    \le
    \mathcal E_\beta[O](f,f)
    \le
    \mathcal E_\beta[1](f,f)
    \le
    \|f\|_{L^2(\pi_\beta)}^2,
\]
which implies
\[
    \langle f,T_\beta[O]f\rangle_{\pi_\beta}\ge0.
\]
Thus $T_\beta[O]$ is positive semidefinite.

Finally, if the Dirichlet form is zero, then $f$ has the same value at both endpoints of every edge with positive stationary flow. Connectivity of the accepted graph therefore implies that $f$ is constant. The converse is immediate.
\end{proof}

For a nontrivial class, the spectral gap has the standard
Dirichlet-form characterization \citep{levin2017markov}
\begin{equation}\label{eq:app-gap-variational}
    \gamma_\beta[O]
    =
    \inf_{\Var_{\pi_\beta}(f)>0}
    \frac{
        \mathcal E_\beta[O](f,f)
    }{
        \Var_{\pi_\beta}(f)
    }.
\end{equation}
By Lemma~\ref{lem:psd}, this is equivalent to the eigenvalue definition
\[
    \gamma_\beta[O]=1-\lambda_2.
\]

\subsection{Proof of Theorem~\ref{thm:mixing}}\label{app:mixing-proof}

\begin{proof}
For a nontrivial connected class, Lemma~\ref{lem:psd} gives
\[
    1=\lambda_1>\lambda_2\ge\cdots\ge\lambda_N\ge0,
    \qquad
    \lambda_2=1-\gamma_\beta[O].
\]

Let
\[
    \psi_k
    =
    \frac{P_k}{\pi_\beta}-1.
\]
By reversibility,
\[
    \psi_k
    =
    T_\beta[O]^k\psi_0.
\]
Expanding the mean-zero function $\psi_0$ in an orthonormal eigenbasis of $T_\beta[O]$ gives
\[
    \chi^2(P_k,\pi_\beta)
    =
    \|\psi_k\|_{L^2(\pi_\beta)}^2
    \le
    (1-\gamma_\beta[O])^{2k}
    \chi^2(P,\pi_\beta).
\]
Cauchy--Schwarz yields
\[
    \TV(P_k,\pi_\beta)
    \le
    \frac12
    \sqrt{\chi^2(P_k,\pi_\beta)},
\]
and
\[
    \chi^2(P,\pi_\beta)
    =
    \sum_y\frac{P(y)^2}{\pi_\beta(y)}-1
    \le
    \pi_{\beta,\min}^{-1}-1.
\]
Using
\[
    (1-\gamma)^k\le e^{-\gamma k}
\]
gives~\eqref{eq:mixing-tv}.

Finally, Markov kernels contract the $\ell_1$ norm of signed measures. Applying this to $P_k-\pi_\beta$ shows that
\[
    \TV(P_{k+1},\pi_\beta)
    \le
    \TV(P_k,\pi_\beta),
\]
and the spectral bound implies convergence to zero.
\end{proof}

\subsection{Generic bounds on the spectral gap}\label{app:mixing-tools}

The main text interprets the spectral gap through acceptance probabilities and sentence dependence. For general admissible graphs, standard conductance and canonical-path arguments provide additional lower bounds.

Define the conductance
\[
    \Phi_\beta[O]
    =
    \min_{0<\pi_\beta(S)\le1/2}
    \frac{
        \sum_{y\in S,\;z\notin S}
        \omega_\beta[O](y,z)
    }{
        \pi_\beta(S)
    }.
\]

For each ordered pair $x\neq z$, choose a directed path $\zeta_{xz}$ along edges with positive flow, and define its congestion by
\[
    \mathcal R_\beta[O;\zeta]
    =
    \max_{e=(u,v)}
    \frac{1}{\omega_\beta[O](u,v)}
    \sum_{(x,z):\,e\in\zeta_{xz}}
    \pi_\beta(x)\pi_\beta(z)|\zeta_{xz}|.
\]

The conductance bounds are the Markov-chain Cheeger inequalities
\citep{jerrum1989approximating,levin2017markov}.
The path bound follows from the canonical-path method
\citep{sinclair1992improved}.

\begin{proposition}[Conductance and canonical paths]\label{prop:app-conductance}
For a connected accepted graph,
\begin{equation}\label{eq:app-cheeger}
    \frac{\Phi_\beta[O]^2}{2}
    \le
    \gamma_\beta[O]
    \le
    2\Phi_\beta[O],
    \qquad
    \gamma_\beta[O]
    \ge
    \frac{1}{\mathcal R_\beta[O;\zeta]}.
\end{equation}

Let
\[
    R
    =
    \max_yL(y)-\min_yL(y),
    \qquad
    B_c
    =
    \max_{\{y,z\}\in E}|L(y)-L(z)|,
\]
\[
    K_{\max}
    =
    \max_i|S_i|,
    \qquad
    o_{\min}
    =
    \min_{\substack{\{y,z\}\in E\\ O(y,z)>0}} O(y,z).
\]
Then
\begin{equation}\label{eq:app-explicit-gap}
    \pi_{\beta,\min}
    \ge
    \frac{e^{-\beta R}}{N},
    \qquad
    \gamma_\beta[O]
    \ge
    \frac{
        o_{\min}e^{-\beta(R+B_c)}
    }{
        NnK_{\max}(N-1)
    }.
\end{equation}
\end{proposition}

\begin{proof}
For the upper conductance bound, substitute $f=\mathds{1}_S$ into~\eqref{eq:app-gap-variational}. Then
\[
    \mathcal E_\beta[O](f,f)
\]
is the stationary flow across the cut, while
\[
    \Var_{\pi_\beta}(\mathds{1}_S)
    =
    \pi_\beta(S)(1-\pi_\beta(S))
    \ge
    \frac12\pi_\beta(S).
\]
This yields $\gamma_\beta[O]\le2\Phi_\beta[O]$.

For the lower bound, the standard level-set argument gives
\[
    \mathcal E_\beta[O](f,f)
    \ge
    \frac{\Phi_\beta[O]^2}{2}
    \Var_{\pi_\beta}(f),
\]
which implies the left inequality in~\eqref{eq:app-cheeger}.

For canonical paths, write $f(z)-f(x)$ as the sum of its increments along $\zeta_{xz}$. Cauchy--Schwarz gives
\[
    (f(z)-f(x))^2
    \le
    |\zeta_{xz}|
    \sum_{(u,v)\in\zeta_{xz}}
    (f(v)-f(u))^2.
\]
Substitution into
\[
    \Var_{\pi_\beta}(f)
    =
    \frac12
    \sum_{x,z}
    \pi_\beta(x)\pi_\beta(z)
    (f(x)-f(z))^2
\]
and collection by directed edge yields
\[
    \Var_{\pi_\beta}(f)
    \le
    \mathcal R_\beta[O;\zeta]\,
    \mathcal E_\beta[O](f,f).
\]

For~\eqref{eq:app-explicit-gap}, shift $L$ by its minimum, which does not change $\pi_\beta$. Then
\[
    \pi_{\beta,\min}\ge\frac{e^{-\beta R}}{N}.
\]
For neighboring states, every conditional proposal probability is at least
\[
    \frac{e^{-\beta B_c}}{K_{\max}}.
\]
Hence every positive directed edge flow is at least
\[
    \frac{
        \pi_{\beta,\min}
        o_{\min}
        e^{-\beta B_c}
    }{
        nK_{\max}
    }.
\]
Choose simple paths of length at most $N-1$. Their congestion is at most $N-1$ divided by the minimum positive edge flow. The canonical-path bound then gives~\eqref{eq:app-explicit-gap}.
\end{proof}

The explicit bound in~\eqref{eq:app-explicit-gap} can be much smaller than the actual spectral gap. Its role is to give a general sufficient bound when sharper structural information is unavailable.

\subsection{Proof of Proposition~\ref{prop:acceptance-comparison}}\label{app:acceptance-proof}

\begin{proof}
By the variational characterization~\eqref{eq:app-gap-variational}, replacing $O_1$ by $O_2$ multiplies every edge contribution to the Dirichlet form by at least $a$. Therefore
\[
    \mathcal E_\beta[O_2](f,f)
    \ge
    a\,\mathcal E_\beta[O_1](f,f)
\]
for every $f$, and hence
\[
    \gamma_\beta[O_2]
    \ge
    a\,\gamma_\beta[O_1].
\]

For constant acceptance $O_\tau(y,z)=\tau$,
\[
    T_\beta[O_\tau]
    =
    (1-\tau)I+\tau K_\beta.
\]
Thus an eigenvalue $\lambda_j(K_\beta)$ is mapped to
\[
    1-\tau+\tau\lambda_j(K_\beta),
\]
which gives
\[
    \gamma_\beta[O_\tau]
    =
    \tau\gamma_\beta[1].
\]
\end{proof}

\subsection{Product-case benchmark}
\label{app:product-case}

The main text uses the factorized-score case as a simple reference for the rewriting rate. Suppose
\[
    L(y)=b+\sum_{i=1}^n \ell_i(y_i),
\]
and at least one sentence position has more than one admissible phrasing. In this case, the stationary distribution factorizes across sentence positions.

\begin{proposition}[Product case]
\label{prop:product-gap}
If all proposals are accepted, then
\[
    \gamma_\beta[1]=\frac1n.
\]
If every proposed rewrite is instead accepted with a constant probability
$\tau\in(0,1]$, then
\[
    \TV\!\left(
        P T_\beta[O_\tau]^k,\pi_\beta
    \right)
    \le
    \sum_{i:|S_i|>1}
    \left(1-\frac{\tau}{n}\right)^k
    \le
    ne^{-\tau k/n}.
\]
\end{proposition}

\begin{proof}
We use the product-chain eigenbasis and the standard coupling bound
for total variation \citep{levin2017markov}.

Under the additive score assumption, $\pi_\beta$ factorizes across sentence positions. Products of constant and mean-zero functions of the individual sentence coordinates therefore form an eigenbasis for $K_\beta$.

If $J$ is the set of coordinates on which such an eigenfunction is nonconstant, its eigenvalue is
\[
    1-\frac{|J|}{n}.
\]
The largest eigenvalue below one is therefore $1-1/n$, which gives
\[
    \gamma_\beta[1]=\frac1n.
\]

For constant acceptance probability $\tau$, couple a chain started from $P$ with one started from $\pi_\beta$ using the same sentence choices, acceptance decisions, and conditional samples. Once a sentence position receives an accepted update, the two chains agree at that position and remain equal there.

The probability that sentence position $i$ receives no accepted update during $k$ attempted rewrites is
\[
    \left(1-\frac{\tau}{n}\right)^k.
\]
A union bound over the nontrivial sentence positions gives
\[
    \TV\!\left(
        P T_\beta[O_\tau]^k,\pi_\beta
    \right)
    \le
    \sum_{i:|S_i|>1}
    \left(1-\frac{\tau}{n}\right)^k.
\]
Finally,
\[
    \left(1-\frac{\tau}{n}\right)^k
    \le
    e^{-\tau k/n},
\]
which yields the stated bound.
\end{proof}

Thus, when sentence scores separate, the natural rewriting time scale is linear in the number of sentence positions. The weak-dependence result in Appendix~\ref{app:weak-dependence} extends this benchmark to settings where sentence scores interact.

\subsection{Comparison across inverse temperatures}\label{app:beta-comparison}

The uniform bound~\eqref{eq:uniform-mixing} requires control of the spectral gap throughout $I$. The following comparison result allows bounds established at selected values of $\beta$ to be extended to nearby parameters.

We compare the stationary measures and Dirichlet forms at the two
parameters, using the comparison method for reversible chains
\citep{diaconis1993comparison}.

\begin{lemma}[Comparison across inverse temperatures]\label{lem:beta-comparison}
Let $\beta,\beta_0\ge0$ and $d=|\beta-\beta_0|$. With $R$ and $B_c$ as defined in Proposition~\ref{prop:app-conductance},
\begin{equation}\label{eq:app-beta-comparison}
    \gamma_\beta[O]
    \ge
    e^{-d(2R+B_c)}
    \gamma_{\beta_0}[O],
    \qquad
    \pi_{\beta,\min}
    \ge
    e^{-dR}
    \pi_{\beta_0,\min}.
\end{equation}
\end{lemma}

\begin{proof}
The density ratio between $\pi_\beta$ and $\pi_{\beta_0}$ lies between
\[
    e^{-dR}
    \quad\text{and}\quad
    e^{dR}.
\]
Similarly, corresponding conditional proposal probabilities differ by at most a factor $e^{\pm dB_c}$. Since the acceptance function is fixed, every positive stationary edge flow at $\beta$ is therefore at least
\[
    e^{-d(R+B_c)}
\]
times the corresponding flow at $\beta_0$.

Moreover,
\[
    \Var_{\pi_\beta}(f)
    =
    \min_b
    \mathbb E_{\pi_\beta}(f-b)^2
    \le
    e^{dR}
    \Var_{\pi_{\beta_0}}(f).
\]
Substituting these two comparisons into the variational characterization~\eqref{eq:app-gap-variational} gives
\[
    \gamma_\beta[O]
    \ge
    e^{-d(2R+B_c)}
    \gamma_{\beta_0}[O].
\]
The minimum-probability bound follows directly from the density-ratio comparison.
\end{proof}

Taking
\[
    d=\max_{\beta\in I}|\beta-\beta_0|
\]
gives bounds valid throughout $I$. Several reference values of $\beta_0$ may be used to obtain sharper uniform bounds.

\subsection{Weakly dependent sentence scores}
\label{app:weak-dependence}

The product case in Proposition~\ref{prop:product-gap} assumes that sentence scores separate. We record here a more general bound allowing the preferred phrasing of one sentence to depend on the others.

For $i\neq j$, define
\[
    C_{ij}(\beta)
    =
    \sup_{y_{-j}=z_{-j}}
    \TV\!\left(
        p_{\beta,i}(\cdot\mid y_{-i}),
        p_{\beta,i}(\cdot\mid z_{-i})
    \right),
\]
and
\[
    \alpha_\beta
    =
    \max_j\sum_{i\neq j}C_{ij}(\beta).
\]
The quantity $C_{ij}(\beta)$ measures how much changing sentence $j$ can alter the conditional update distribution for sentence $i$, while $\alpha_\beta$ bounds the total influence of one sentence on the others.

The condition $\alpha_\beta<1$ is a Dobrushin-type influence condition
\citep{dyer2009matrix}.
The following bounds follow from path coupling
\citep{bubley1997path}.

\begin{proposition}[Weak dependence]
\label{prop:weak-dependence}
If $\alpha_\beta<1$, then
\[
    \TV(PK_\beta^k,\pi_\beta)
    \le
    n\left(
        1-\frac{1-\alpha_\beta}{n}
    \right)^k,
    \qquad
    \gamma_\beta[1]
    \ge
    \frac{1-\alpha_\beta}{n}.
\]
If, more generally, $O(y,z)\ge a>0$ on every product edge, then
\[
    \gamma_\beta[O]
    \ge
    \frac{a(1-\alpha_\beta)}{n}.
\]
For constant acceptance probability $\tau$, the total-variation bound above holds with $1-\alpha_\beta$ replaced by $\tau(1-\alpha_\beta)$.
\end{proposition}

The condition $\alpha_\beta<1$ is a standard weak-dependence requirement. When $\alpha_\beta$ is small, the $1/n$ mixing scale from the product case is only mildly degraded; as $\alpha_\beta$ approaches one, sentence interactions can substantially slow convergence.

\begin{proof}
Start from two states that differ only at coordinate $j$. Use the same update coordinate in both chains and a maximal coupling of the corresponding conditional proposal distributions.

If coordinate $j$ is updated, the disagreement is removed. If coordinate $i\neq j$ is updated, an additional disagreement is created with probability at most $C_{ij}(\beta)$. Hence the expected Hamming distance contracts by at least the factor
\[
    1-\frac{1-\alpha_\beta}{n}.
\]

Extending this coupling along a shortest Hamming path and iterating gives
\[
    \TV(PK_\beta^k,\pi_\beta)
    \le
    n\left(
        1-\frac{1-\alpha_\beta}{n}
    \right)^k.
\]
The same contraction bound applied to nonconstant eigenfunctions gives
\[
    \gamma_\beta[1]
    \ge
    \frac{1-\alpha_\beta}{n}.
\]
Finally, Proposition~\ref{prop:acceptance-comparison} yields the bound for $O\ge a$, while constant acceptance $\tau$ changes the contraction factor to
\[
    1-\frac{\tau(1-\alpha_\beta)}{n}.
\]
\end{proof}

\section{Technical Results for Human Calibration}
\label{app:calibration}

\subsection{Score drift}\label{app:score-drift}

For a content class $c$, define the expected score change after one attempted update by
\[
    d_{c,\beta}[O](y)
    =
    \bigl(T_{c,\beta}[O]L_c\bigr)(y)-L_c(y).
\]
When $O\equiv1$, we write $d_{c,\beta}=d_{c,\beta}[1]$.

The identity below follows from the standard Dirichlet-form identity
for reversible Markov chains \citep{levin2017markov}.

\begin{lemma}[Drift identity]\label{lem:drift-identity}
For any distribution $P$ on $\mathcal A_c$, let
\[
    \psi=\frac{P}{\pi_{c,\beta}}.
\]
Then
\[
    \mathbb E_P d_{c,\beta}[O]
    =
    -\mathcal E_{c,\beta}[O](\psi,L_c).
\]
\end{lemma}

\begin{proof}
Suppress the class index. Expanding the expectation gives
\[
    \sum_{y,z}
    \psi(y)\omega_\beta[O](y,z)
    \bigl(L(z)-L(y)\bigr).
\]
Exchange $y$ and $z$ and use the symmetry of the stationary flow $\omega_\beta[O]$. Averaging the two expressions gives
\[
    -\mathcal E_\beta[O](\psi,L),
\]
as claimed.
\end{proof}

The drift identity gives an additional interpretation of the rewriting process. Under Assumption~\ref{ass:human-model}, human writing is stationary at $\beta_H$ and therefore has zero expected score drift.

The following stronger statement about the initial machine distribution is not required for the later convergence results, but records when machine text has a systematic drift at the human parameter.

\begin{assumption}[Monotone likelihood ratio]\label{ass:mlr}
For the class under consideration, there exists a nondecreasing function $g_c$ on $L_c(\mathcal A_c)$ such that
\[
    \frac{P_M^c(y)}{\pi_{c,\beta_H}(y)}
    =
    g_c(L_c(y)).
\]
\end{assumption}

\begin{theorem}[Expected score change]\label{thm:expected-drift}
Under Assumption~\ref{ass:human-model},
\[
    P_H^cT_{c,\beta_H}[O]=P_H^c,
    \qquad
    \mathbb E_{P_H^c}
    d_{c,\beta_H}[O]
    =
    0.
\]
Under Assumption~\ref{ass:mlr},
\[
    \mathbb E_{P_M^c}
    d_{c,\beta_H}[O]
    \le0.
\]
The inequality is strict if an accepted edge $\{y,z\}$ with positive weight satisfies
\[
    \bigl(g_c(L_c(z))-g_c(L_c(y))\bigr)
    \bigl(L_c(z)-L_c(y)\bigr)>0.
\]

More generally, if $Y_0\sim\pi_{c,\beta'}$, one $T_{c,\beta}[O]$ step gives
\[
    \mathbb E_{\pi_{c,\beta'}}
    d_{c,\beta}[O]
    =
    \frac12
    \mathbb E
    \left[
        \Delta
        \left(
            1-e^{-(\beta-\beta')\Delta}
        \right)
    \right],
\]
where
\[
    \Delta=L_c(Y_1)-L_c(Y_0).
\]
If the accepted graph is connected and $L_c$ is nonconstant, this expectation has the same sign as $\beta-\beta'$.
\end{theorem}

\begin{proof}
Stationarity gives the human statement. Under Assumption~\ref{ass:mlr}, every contribution to
\[
    \mathcal E_{c,\beta_H}[O]
    (g_c\circ L_c,L_c)
\]
is nonnegative, and Lemma~\ref{lem:drift-identity} gives the result.

For the final identity, suppress the class index and set
\[
    u=\beta-\beta'.
\]
Since
\[
    \frac{\pi_{\beta'}(y)}{\pi_\beta(y)}
    \propto e^{-uL(y)},
\]
we have
\[
    \psi(z)-\psi(y)
    =
    \psi(y)
    \left(
        e^{-u(L(z)-L(y))}-1
    \right).
\]
Substitution into Lemma~\ref{lem:drift-identity} yields the stated expression. Its sign follows because, for $u\neq0$ and $\Delta\neq0$,
\[
    \Delta(1-e^{-u\Delta})
\]
has the sign of $u$.
\end{proof}

Assumption~\ref{ass:mlr} concerns only the initial machine distribution and is not used in the finite-sample convergence results.

\subsection{Pseudolikelihood interpretation}
\label{app:pseudolikelihood-proof}

The estimator in Section~\ref{sec:pseudolikelihood} also admits an interpretation in terms of how sentence resampling changes the complete-text score.

For a text $h\in\mathcal A_c$, define
\[
    d_{c,\beta}(h)
    =
    \frac1{n_c}
    \sum_{i=1}^{n_c}
    \left[
        \sum_{s\in S_{c,i}}
        p_{c,\beta,i}(s\mid h_{-i})
        L_c(h^{i\leftarrow s})
        -
        L_c(h)
    \right].
\]
This is the expected change in $L_c$ when one sentence position is selected uniformly and resampled according to the conditional distribution $p_{c,\beta,i}$.

Averaging over the human calibration sample gives
\[
    \widehat D_m(\beta)
    =
    \frac1m
    \sum_{j=1}^m
    d_{C_j,\beta}(H_j).
\]

The derivatives follow from the standard identities for the
log-partition function of an exponential family
\citep{wainwright2008graphical}.

\begin{proposition}[Pseudolikelihood derivatives]
\label{prop:pseudolikelihood-derivatives}
The function $\ell_m$ is concave and satisfies
\[
    \ell_m'(\beta)
    =
    -\widehat D_m(\beta),
    \qquad
    \ell_m''(\beta)
    =
    -\frac1m
    \sum_{j=1}^m
    v_{C_j,\beta}(H_j),
\]
where $v_{c,\beta}$ is defined in Section~\ref{sec:pseudolikelihood}. Consequently, every interior maximizer satisfies
\[
    \widehat D_m(\widehat\beta_m)=0.
\]
\end{proposition}

\begin{proof}
For one conditional log-probability,
\[
\partial_\beta
\log p_{c,\beta,i}(h_i\mid h_{-i})
=
L_c(h)
-
\sum_{s\in S_{c,i}}
p_{c,\beta,i}(s\mid h_{-i})
L_c(h^{i\leftarrow s}).
\]
Averaging over sentence positions and calibration texts gives
\[
    \ell_m'(\beta)
    =
    -\widehat D_m(\beta).
\]

Differentiating once more gives
\[
\partial_\beta^2
\log p_{c,\beta,i}(h_i\mid h_{-i})
=
-
\Var_{s\sim p_{c,\beta,i}(\cdot\mid h_{-i})}
\left[
    L_c(h^{i\leftarrow s})
\right].
\]
Therefore
\[
    \ell_m''(\beta)
    =
    -\frac1m
    \sum_{j=1}^m
    v_{C_j,\beta}(H_j)
    \le 0,
\]
so $\ell_m$ is concave.
\end{proof}

Thus, when the maximizer lies in the interior of $I$, $\widehat\beta_m$ is equivalently the value of $\beta$ for which resampling sentences produces no average change in score on the human calibration texts.

The estimator uses the conditional distributions $p_{c,\beta,i}$ before applying the acceptance rule $O_c$. Since every fixed symmetric acceptance rule preserves the same stationary family $\pi_{c,\beta}$, the estimator does not depend on which such acceptance rule is used.

\subsection{Why zero average score change is not sufficient}
\label{app:zero-drift}

Zero average score change does not by itself imply stationarity. Consider a one-sentence content class with three admissible alternatives having scores
\[
    -2,\qquad -1,\qquad 0.
\]
At $\beta=0$, the conditional resampling distribution is uniform. A distribution concentrated entirely on the middle-scoring alternative has zero expected score change, but it is not the uniform stationary distribution.

Thus the condition
\[
    \widehat D_m(\beta)=0
\]
does not by itself identify a stationary distribution. Its interpretation as identifying the human parameter relies on Assumption~\ref{ass:human-model}.

\subsection{Population slope and identifiability}
\label{app:positive-slope}

Define the population proposal drift
\[
    D_\nu(\beta)
    =
    \mathbb E_{C\sim\nu,\;H\sim P_H^C}
    d_{C,\beta}(H).
\]

\begin{lemma}[Positive population slope]\label{lem:positive-slope}
Under Assumptions~\ref{ass:human-model} and~\ref{ass:bounded-scores},
\[
    D_\nu(\beta_H)=0,
    \qquad
    D_\nu'(\beta)
    =
    \mathbb E\,v_{C,\beta}(H),
\]
and
\[
    \kappa
    =
    \inf_{\beta\in I}
    D_\nu'(\beta)
    >0.
\]
\end{lemma}

\begin{proof}
Stationarity under Assumption~\ref{ass:human-model} gives
\[
    D_\nu(\beta_H)=0.
\]
The bounds
\[
    |d_{c,\beta}(h)|\le B,
    \qquad
    v_{c,\beta}(h)\le B^2/4
\]
allow differentiation under the expectation, giving
\[
    D_\nu'(\beta)
    =
    \mathbb E\,v_{C,\beta}(H).
\]

If $L_c$ is nonconstant, connectivity implies that some adjacent states have different scores. Every state has positive probability under $\pi_{c,\beta_H}$, and every sentence alternative has positive conditional probability for $\beta\in I$. Hence
\[
    \mathbb E_{P_H^c}v_{c,\beta}(H)>0
\]
for each such class. Assumption~\ref{ass:bounded-scores} assigns positive $\nu$-probability to at least one such class, so
\[
    \mathbb E\,v_{C,\beta}(H)>0
\]
for every $\beta\in I$. Continuity and compactness of $I$ then imply $\kappa>0$.
\end{proof}

For a single calibration class,
\[
    D_\nu'(\beta_H)
    =
    \mathcal E_{c,\beta_H}[1](L_c,L_c)
    \ge
    \gamma_{c,\beta_H}[1]\,
    \Var_{\pi_{c,\beta_H}}(L_c).
\]
This relates local identifiability of $\beta_H$ to the spectral gap and score variability. A bound on $\kappa$, however, must hold throughout $I$.

\subsection{Proof of Theorem~\ref{thm:beta-estimation}}
\label{app:beta-estimation-proof}

\begin{proof}
Fix $\epsilon>0$ with
\[
    \beta_H-\epsilon,\,
    \beta_H+\epsilon
    \in I.
\]
By Lemma~\ref{lem:positive-slope},
\[
    D_\nu(\beta_H-\epsilon)<0
    <
    D_\nu(\beta_H+\epsilon).
\]
The strong law of large numbers implies that $\widehat D_m$ eventually has the same signs at these two points almost surely. Since $\ell_m$ is concave and
\[
    \ell_m'=-\widehat D_m,
\]
every maximizer eventually lies inside this interval. Applying the argument to a sequence $\epsilon\downarrow0$ proves
\[
    \widehat\beta_m\to\beta_H
    \qquad\text{almost surely}.
\]

For the finite-sample bound, let
\[
    a=\beta_H+t.
\]
If a maximizer lies to the right of $a$, concavity implies
\[
    \widehat D_m(a)\le0.
\]
By Lemma~\ref{lem:positive-slope},
\[
    D_\nu(a)\ge\kappa t.
\]
Each summand defining $\widehat D_m(a)$ lies in $[-B,B]$.
Hoeffding's inequality \citep{hoeffding1963probability} therefore gives
\[
    \Pr\{\widehat\beta_m-\beta_H>t\}
    \le
    \exp\!\left(
        -\frac{m\kappa^2t^2}{2B^2}
    \right).
\]
The same argument for the lower tail yields
\[
    \Pr\{
        |\widehat\beta_m-\beta_H|>t
    \}
    \le
    2\exp\!\left(
        -\frac{m\kappa^2t^2}{2B^2}
    \right).
\]
Setting the right-hand side equal to $\delta$ gives~\eqref{eq:beta-error}.

Integrating the tail bound also yields
\[
    \mathbb E
    |\widehat\beta_m-\beta_H|^2
    \le
    \frac{4B^2}{m\kappa^2}.
\]
\end{proof}

\subsection{Sensitivity to the parameter}
\label{app:sensitivity-proof}

\begin{proof}[Proof of Lemma~\ref{lem:parameter-sensitivity}]
Differentiating the stationary probability gives
\[
    \partial_\beta\pi_{c,\beta}(y)
    =
    \pi_{c,\beta}(y)
    \left(
        L_c(y)
        -
        \mathbb E_{\pi_{c,\beta}}L_c
    \right).
\]
Therefore
\[
\begin{aligned}
    \TV(\pi_{c,a},\pi_{c,b})
    &\le
    \frac12
    \int_a^b
    \sum_y
    \left|
        \partial_\beta
        \pi_{c,\beta}(y)
    \right|
    d\beta\\
    &\le
    \frac12
    \int_a^b
    \sqrt{
        \Var_{\pi_{c,\beta}}(L_c)
    }
    d\beta\\
    &\le
    \frac{\sigma_c}{2}|a-b|.
\end{aligned}
\]

The KL bound below uses the log-partition representation of
relative entropy in an exponential family
\citep{wainwright2008graphical}.

For completeness, the corresponding KL bound is
\[
    \KL(\pi_{c,a}\Vert\pi_{c,b})
    \le
    \frac{\sigma_c^2}{2}(a-b)^2.
\]
To see this, let
\[
    F(\beta)=\log Z_c(\beta).
\]
Then
\[
    F''(\beta)
    =
    \Var_{\pi_{c,\beta}}(L_c)
    \le\sigma_c^2,
\]
while
\[
    \KL(\pi_{c,a}\Vert\pi_{c,b})
    =
    F(b)-F(a)-(b-a)F'(a).
\]
The result follows from the integral remainder in Taylor's theorem.
\end{proof}

\section{Technical Results for Finite-Sample Convergence}
\label{app:finite-sample}

Throughout this appendix, fix the target content class from Section~\ref{sec:finite-sample-convergence} and continue to suppress its class index.

\subsection{Proof of Theorem~\ref{thm:finite-sample}}
\label{app:finite-sample-proof}

\begin{proof}
By Assumption~\ref{ass:human-model},
\[
    P_H=\pi_{\beta_H}.
\]
Using the triangle inequality through $\pi_{\widehat\beta_m}$,
\[
\begin{aligned}
\TV(Q_{m,k},P_H)
&\le
\TV\!\left(
    P_M T_{\widehat\beta_m}[O]^k,
    \pi_{\widehat\beta_m}
\right)
+
\TV\!\left(
    \pi_{\widehat\beta_m},
    \pi_{\beta_H}
\right).
\end{aligned}
\]
The uniform mixing bound~\eqref{eq:uniform-mixing} gives
\[
    \TV\!\left(
        P_M T_{\widehat\beta_m}[O]^k,
        \pi_{\widehat\beta_m}
    \right)
    \le
    M_k[O].
\]
Lemma~\ref{lem:parameter-sensitivity} gives
\[
    \TV\!\left(
        \pi_{\widehat\beta_m},
        \pi_{\beta_H}
    \right)
    \le
    \frac{\sigma}{2}
    |\widehat\beta_m-\beta_H|.
\]
Applying Theorem~\ref{thm:beta-estimation} yields~\eqref{eq:finite-sample}.

The calibration event is independent of $k$, so the same event gives the bound simultaneously for every $k$.

Finally, for any deterministic sequence $k_m\to\infty$,
\[
    M_{k_m}[O]\to0,
\]
while
\[
    \widehat\beta_m\to\beta_H
    \qquad\text{almost surely}.
\]
The same triangle inequality therefore gives
\[
    \TV(Q_{m,k_m},P_H)\to0
    \qquad\text{almost surely}.
\]
\end{proof}

\subsection{Effect of one rewrite on human writing}
\label{app:human-stability}

The same argument also controls how much the estimated rewriting kernel can move the human distribution itself.

\begin{corollary}\label{cor:human-stability}
On the calibration event of Theorem~\ref{thm:finite-sample},
\[
    \TV\!\left(
        P_H T_{\widehat\beta_m}[O],
        P_H
    \right)
    \le
    \sigma\,t_m(\delta).
\]
\end{corollary}

\begin{proof}
Insert $\pi_{\widehat\beta_m}$ between the two distributions. Since
$\pi_{\widehat\beta_m}$ is stationary for
$T_{\widehat\beta_m}[O]$ and Markov kernels contract total variation,
\[
\begin{aligned}
\TV\!\left(
    P_H T_{\widehat\beta_m}[O],
    P_H
\right)
&\le
2\TV\!\left(
    P_H,
    \pi_{\widehat\beta_m}
\right)\\
&\le
\sigma
|\widehat\beta_m-\beta_H|.
\end{aligned}
\]
Apply Theorem~\ref{thm:beta-estimation}.
\end{proof}

\subsection{Concentration of an empirical distribution}
\label{app:empirical-concentration}

The following result is used to control the error from observing only finitely many rewritten outputs.

\begin{lemma}[Empirical distribution concentration]
\label{lem:empirical-concentration}
Let $Z_1,\ldots,Z_p$ be independent random variables taking values in a set of size $d$. Their distributions need not be identical. Define
\[
    \widehat P_p
    =
    \frac1p\sum_{j=1}^p\delta_{Z_j},
    \qquad
    \overline P_p
    =
    \frac1p\sum_{j=1}^p\mathcal L(Z_j).
\]
Then, with probability at least $1-\delta$,
\begin{equation}\label{eq:empirical-concentration}
    \TV(\widehat P_p,\overline P_p)
    \le
    c_d(p,\delta)
    :=
    \frac{
        \sqrt d+\sqrt{2\log(1/\delta)}
    }{
        2\sqrt p
    }.
\end{equation}
\end{lemma}

\begin{proof}
Let
\[
    p_{j,a}=\Pr(Z_j=a).
\]
For each state $a$,
\[
    \Var(\widehat P_p(a))
    =
    \frac1{p^2}
    \sum_{j=1}^p
    p_{j,a}(1-p_{j,a})
    \le
    \frac{
        \overline P_p(a)
        (1-\overline P_p(a))
    }{p}.
\]
Hence
\[
\mathbb E\,
\TV(\widehat P_p,\overline P_p)
\le
\frac1{2\sqrt p}
\sum_{a=1}^d
\sqrt{
    \overline P_p(a)
    (1-\overline P_p(a))
}.
\]
By Cauchy--Schwarz, the sum is at most $\sqrt d$.

Changing one observation changes the total-variation distance by
at most $1/p$. McDiarmid's bounded-differences inequality
\citep{mcdiarmid1989bounded} therefore adds at most
\[
    \sqrt{\frac{\log(1/\delta)}{2p}}
    =
    \frac{\sqrt{2\log(1/\delta)}}{2\sqrt p}
\]
with probability at least $1-\delta$, proving~\eqref{eq:empirical-concentration}.
\end{proof}

A sharper, distribution-dependent version replaces $\sqrt d$ by
\[
    s(P)
    =
    \sum_{a=1}^d
    \sqrt{P(a)(1-P(a))},
    \qquad
    s(P)^2\le d-1,
\]
with $P=\overline P_p$.

\subsection{Proof of Theorem~\ref{thm:empirical-output}}
\label{app:empirical-output-proof}

Conditional on the calibration sample, let the $i$th rewritten output have distribution
\[
    Q_{i,m,k}
    =
    P_{M,i}T_{\widehat\beta_m}[O]^k,
\]
where the initial distribution $P_{M,i}$ may be a point mass. Define their average distribution by
\[
    \overline Q_{m,k,p}
    =
    \frac1p\sum_{i=1}^p Q_{i,m,k}.
\]

\begin{proof}[Proof of Theorem~\ref{thm:empirical-output}]
Conditional on the calibration sample, apply Lemma~\ref{lem:empirical-concentration} to
\[
    \phi(Y_1^{(k)}),\ldots,\phi(Y_p^{(k)})
\]
with failure probability $\delta/2$. This gives
\[
    \TV\!\left(
        \widehat Q_{m,k,p}^{\,\phi},
        \overline Q_{m,k,p}^{\,\phi}
    \right)
    \le
    c_d(p,\delta/2).
\]

Apply Theorem~\ref{thm:finite-sample} with failure probability $\delta/2$. Since total variation is convex and cannot increase after applying $\phi$,
\[
\TV\!\left(
    \overline Q_{m,k,p}^{\,\phi},
    P_H^\phi
\right)
\le
M_k[O]
+
\frac{\sigma B}{2\kappa}
\sqrt{
    \frac{2\log(4/\delta)}{m}
}.
\]
Combining the two bounds with the triangle inequality and a union bound gives~\eqref{eq:empirical-output}.
\end{proof}

The distribution-dependent quantity $s(\overline Q_{m,k,p}^\phi)$ may replace $\sqrt d$ in the sampling term.

\subsection{Exact sufficient update and sample sizes}
\label{app:sample-bounds}

Let
\[
\Gamma=\Gamma_{c,I}[O],
\qquad
A=\frac12\sqrt{\pi_{c,I,*}^{-1}-1},
\qquad
\sigma=\sigma_c,
\]
and define
\[
\log_+(u)=\max\{0,\log u\}.
\]

For $0<\epsilon,\delta<1$, it is sufficient for Theorem~\ref{thm:finite-sample} to take
\begin{equation}\label{eq:exact-km}
    k
    \ge
    \left\lceil
        \frac1\Gamma
        \log^+\frac{2A}{\epsilon}
    \right\rceil,
    \qquad
    m
    \ge
    \frac{
        2\sigma^2B^2
    }{
        \kappa^2\epsilon^2
    }
    \log\frac{2}{\delta}.
\end{equation}
These choices make the rewriting and calibration terms in~\eqref{eq:finite-sample} at most $\epsilon/2$ each.

For Theorem~\ref{thm:empirical-output}, it is sufficient to take
\begin{align}
    k
    &\ge
    \left\lceil
        \frac1\Gamma
        \log^+\frac{3A}{\epsilon}
    \right\rceil,
    \label{eq:exact-k}\\
    m
    &\ge
    \frac{
        9\sigma^2B^2
    }{
        2\kappa^2\epsilon^2
    }
    \log\frac{4}{\delta},
    \label{eq:exact-m}\\
    p
    &\ge
    \frac{9}{4\epsilon^2}
    \left(
        \sqrt d
        +
        \sqrt{2\log(2/\delta)}
    \right)^2.
    \label{eq:exact-p}
\end{align}
These choices make the three terms in~\eqref{eq:empirical-output} at most $\epsilon/3$ each.

\subsection{Computation and score evaluations}
\label{app:computation}

The objective $\ell_m$ is concave and
\[
    \ell_m'(\beta)=-\widehat D_m(\beta),
\]
so the sign of $\widehat D_m$ determines the direction in which the maximizer lies. If the maximizer is in the interior of $I$, bisection can locate it. Otherwise, the maximizing endpoint is used.

To obtain parameter accuracy $\xi>0$, bisection requires
\[
    O\!\left(
        \log
        \frac{\beta_+-\beta_-}{\xi}
    \right)
\]
iterations. Replacing $\widehat\beta_m$ by an estimate within $\xi$ adds at most
\[
    \frac{\sigma\xi}{2}
\]
to the total-variation error by Lemma~\ref{lem:parameter-sensitivity}.

For each calibration text, the score of the complete text must be evaluated for every candidate replacement at every sentence position. Without reusing evaluations, this requires at most
\[
    \sum_{j=1}^m
    \sum_{i=1}^{n_{C_j}}
    |S_{C_j,i}|
\]
complete-text score evaluations.

If
\[
    K_{\max}
    =
    \max_{c,i}|S_{c,i}|,
\]
then $p$ independent chains run for $k$ attempted rewrites require at most
\[
    pkK_{\max}
\]
complete-text score evaluations and at most $pk$ evaluations of the acceptance rule. These counts exclude the cost of constructing and verifying the admissible candidate sets.

\section{Technical Results for Direct Kernel Learning}
\label{app:direct-kernel}

\subsection{Total-variation contraction}
\label{app:tv-contraction}

Throughout this appendix, we fix the target content class and suppress
its index, writing $\Gamma_I[O]=\Gamma_{c,I}[O]$ and
$\pi_{I,*}=\pi_{c,I,*}$.

For a finite-state Markov kernel $K$, the Dobrushin contraction
coefficient \citep{seneta1979coefficients} is
\[
    \vartheta(K)
    =
    \max_{y,z}
    \TV\!\left(
        K(y,\cdot),
        K(z,\cdot)
    \right).
\]
Equivalently,
\[
    \vartheta(K)
    =
    1-
    \min_{y,z}
    \sum_u
    \min\{K(y,u),K(z,u)\}.
\]

For any distributions $P,Q$,
\begin{equation}\label{eq:app-tv-contraction}
    \TV(PK,QK)
    \le
    \vartheta(K)\TV(P,Q).
\end{equation}
The coefficient is sharp: equality in the best possible constant follows by taking $P$ and $Q$ to be point masses at states whose transition rows attain the maximum.

For the weighted rewriting family from Section~\ref{sec:convergence}, define
\[
    \rho_\ell[O;I]
    =
    \sup_{\beta\in I}
    \vartheta\!\left(
        T_\beta[O]^\ell
    \right).
\]
The uniform mixing bound~\eqref{eq:uniform-mixing} gives
\[
    \rho_\ell[O;I]
    \le
    \min\left\{
        1,\,
        \sqrt{\pi_{I,*}^{-1}-1}\,
        e^{-\Gamma_I[O]\ell}
    \right\}
    \le
    \min\left\{
        1,\,
        \pi_{I,*}^{-1/2}
        e^{-\Gamma_I[O]\ell}
    \right\}.
\]
Thus sufficiently many rewriting steps give a uniform contraction factor strictly below one.

For example, when the state space contains at least two states, taking
\[
    \ell
    \ge
    \max\left\{
        1,\,
        \left\lceil
            \frac1{\Gamma_I[O]}
            \log
            \frac{2}{\sqrt{\pi_{I,*}}}
        \right\rceil
    \right\}
\]
ensures
\[
    \rho_\ell[O;I]\le\frac12.
\]

\subsection{Contraction from a common component}
\label{app:minorization}

A uniform minorization condition gives the following contraction
bound \citep{rosenthal1995minorization}.

\begin{proposition}[Common-component contraction]
\label{prop:minorization}
Suppose that for some $\ell\ge1$,
\[
    K^\ell(y,\cdot)
    =
    \lambda\zeta(\cdot)
    +(1-\lambda)\mu_y(\cdot),
    \qquad
    0<\lambda\le1,
\]
where $\zeta$ and each $\mu_y$ are probability distributions. Then
\[
    \vartheta(K^\ell)\le1-\lambda.
\]
\end{proposition}

\begin{proof}
The common term $\lambda\zeta$ cancels when two output distributions are subtracted. Therefore
\[
\TV(PK^\ell,QK^\ell)
\le
(1-\lambda)\TV(P,Q),
\]
which proves the result.
\end{proof}

\subsection{General convergence with approximate preservation}
\label{app:direct-general-convergence}

We record here the more general form of Theorem~\ref{thm:approx-preservation}, allowing different admissible sets and rewriting kernels for different initial machine texts.

Let $y_1,\ldots,y_p$ be fixed initial machine texts. For each $i$, let $\mathcal A_i$ be a finite admissible set containing $y_i$ and the support of $\widehat P_{H,m}$, and let $T_{i,m}$ be a learned Markov kernel on $\mathcal A_i$.

Suppose that for some $\ell_m\ge1$ and $0\le\rho_m<1$,
\[
    \TV(
        PT_{i,m}^{\ell_m},
        QT_{i,m}^{\ell_m}
    )
    \le
    \rho_m\TV(P,Q)
\]
for all distributions $P,Q$ supported on $\mathcal A_i$ and all $i$.

Define
\[
    \epsilon_m
    =
    \max_i
    \TV(
        \widehat P_{H,m}T_{i,m},
        \widehat P_{H,m}
    ),
\]
\[
    P_{i,k}
    =
    \delta_{y_i}T_{i,m}^k,
    \qquad
    \overline P_{M,k}
    =
    \frac1p\sum_{i=1}^pP_{i,k}.
\]

\begin{theorem}[General convergence with preservation error]
\label{thm:general-approx-preservation}
Write
\[
    k=q\ell_m+r,
    \qquad
    0\le r<\ell_m.
\]
Then
\begin{align}
\TV(P_{i,k},\widehat P_{H,m})
&\le
\rho_m^q
\TV(\delta_{y_i},\widehat P_{H,m})
\nonumber\\
&\quad+
\left(
    \ell_m
    \frac{1-\rho_m^q}{1-\rho_m}
    +r
\right)
\epsilon_m.
\label{eq:general-approx-preservation}
\end{align}
Consequently,
\[
\TV(
    \overline P_{M,k},
    \widehat P_{H,m}
)
\le
\rho_m^q\Delta_0
+
\left(
    \ell_m
    \frac{1-\rho_m^q}{1-\rho_m}
    +r
\right)
\epsilon_m,
\]
where
\[
    \Delta_0
    =
    \frac1p
    \sum_{i=1}^p
    \TV(
        \delta_{y_i},
        \widehat P_{H,m}
    )
    \le1.
\]
\end{theorem}

\begin{proof}
For any integer $s\ge1$,
\[
\TV(
    \widehat P_{H,m}T_{i,m}^s,
    \widehat P_{H,m}
)
\le
s\epsilon_m,
\]
by the triangle inequality and contraction of total variation under a Markov kernel.

Let
\[
    \Delta_{i,k}
    =
    \TV(
        P_{i,k},
        \widehat P_{H,m}
    ).
\]
After one block of $\ell_m$ updates,
\[
    \Delta_{i,(q+1)\ell_m}
    \le
    \rho_m\Delta_{i,q\ell_m}
    +
    \ell_m\epsilon_m.
\]
Iterating gives
\[
    \Delta_{i,q\ell_m}
    \le
    \rho_m^q\Delta_{i,0}
    +
    \ell_m
    \frac{1-\rho_m^q}{1-\rho_m}
    \epsilon_m.
\]
The remaining $r$ steps add at most $r\epsilon_m$. Averaging over $i$ and using convexity of total variation gives the final statement.
\end{proof}

Theorem~\ref{thm:approx-preservation} is the special case of a common admissible space, a common learned kernel, and $r=0$.

\subsection{Proof of Theorem~\ref{thm:human-preservation}}
\label{app:human-preservation-proof}

\begin{proof}
By the definition of $\widehat T_m$,
\[
\TV(
    \widehat P_{H,m}\widehat T_m,
    \widehat P_{H,m}
)
\le
\TV(
    \widehat P_{H,m}T^\star,
    \widehat P_{H,m}
).
\]
Using the triangle inequality,
\[
\begin{aligned}
\TV(
    \widehat P_{H,m}T^\star,
    \widehat P_{H,m}
)
&\le
\TV(
    \widehat P_{H,m}T^\star,
    P_HT^\star
)
\\
&\quad+
\TV(
    P_HT^\star,
    P_H
)
+
\TV(
    P_H,
    \widehat P_{H,m}
).
\end{aligned}
\]
Since $P_HT^\star=P_H$ and Markov kernels contract total variation,
\[
    \epsilon_m
    \le
    2\TV(
        \widehat P_{H,m},
        P_H
    ).
\]
Lemma~\ref{lem:empirical-concentration}, applied with $d=N$ and $p=m$, gives
\[
    \TV(
        \widehat P_{H,m},
        P_H
    )
    \le
    \frac{
        \sqrt N+\sqrt{2\log(1/\delta)}
    }{
        2\sqrt m
    }
\]
with probability at least $1-\delta$. Multiplying by two gives~\eqref{eq:human-preservation-error}.
\end{proof}

If the candidate set contains only a kernel satisfying
\[
    \TV(P_HT^\star,P_H)\le\eta_{\rm st},
\]
the same argument gives
\[
    \epsilon_m
    \le
    2\TV(
        \widehat P_{H,m},
        P_H
    )
    +
    \eta_{\rm st}.
\]
Thus an imperfect candidate family introduces an additional error that does not vanish with more human samples.

The same proof extends to an infinite candidate collection whenever a measurable minimizer exists.

\subsection{Sampling error and proof of Theorem~\ref{thm:direct-finite-sample}}
\label{app:direct-finite-sample-proof}

Conditional on the human sample, the rewritten outputs are independent. Lemma~\ref{lem:empirical-concentration} therefore gives, with failure probability at most $\delta/2$,
\[
\TV\!\left(
    \widehat P_{M,k,p}^{\,\phi},
    (P_M\widehat T_m^k)^\phi
\right)
\le
\frac{
    \sqrt d+\sqrt{2\log(2/\delta)}
}{
    2\sqrt p
}.
\]

With probability at least $1-\delta/2$,
\[
\TV(
    \widehat P_{H,m},
    P_H
)
\le
\frac{
    \sqrt N+\sqrt{2\log(2/\delta)}
}{
    2\sqrt m
}.
\]
On this event, Theorem~\ref{thm:human-preservation} gives
\[
    \epsilon_m
    \le
    \frac{
        \sqrt N+\sqrt{2\log(2/\delta)}
    }{
        \sqrt m
    }.
\]

For $k=q\ell$, Theorem~\ref{thm:approx-preservation} therefore yields
\[
\TV(
    P_M\widehat T_m^k,
    \widehat P_{H,m}
)
\le
\rho^q
+
\frac{\ell}{1-\rho}
\frac{
    \sqrt N+\sqrt{2\log(2/\delta)}
}{
    \sqrt m
}.
\]

Applying $\phi$ cannot increase total variation. Using the triangle inequality through
\[
    (P_M\widehat T_m^k)^\phi
    \quad\text{and}\quad
    \widehat P_{H,m}^{\,\phi},
\]
and then adding the human sampling error gives
\begin{align*}
\TV\!\left(
    \widehat P_{M,k,p}^{\,\phi},
    P_H^\phi
\right)
&\le
\rho^q\\
&\quad+
\left(
    \frac{2\ell}{1-\rho}+1
\right)
\frac{
    \sqrt N+\sqrt{2\log(2/\delta)}
}{
    2\sqrt m
}\\
&\quad+
\frac{
    \sqrt d+\sqrt{2\log(2/\delta)}
}{
    2\sqrt p
}.
\end{align*}
A union bound gives probability at least $1-\delta$, proving Theorem~\ref{thm:direct-finite-sample}.

\subsection{Comparison with the empirical human distribution}
\label{app:empirical-human-comparison}

If the goal is only to compare rewritten outputs with the empirical human distribution rather than the true $P_H$, the human sampling term is unnecessary. In the general setting of Theorem~\ref{thm:general-approx-preservation}, with probability at least $1-\delta$,
\[
\TV\!\left(
    \widehat P_{M,k,p}^{\,\phi},
    \widehat P_{H,m}^{\,\phi}
\right)
\le
b_{k,m,p}
+
c_d(p,\delta),
\]
where $b_{k,m,p}$ denotes the bound from Theorem~\ref{thm:general-approx-preservation}. Comparison with the true human distribution adds the human sampling term $c_d(m,\delta)$.

\subsection{Representations and lumpability}
\label{app:lumpability}

The map $\phi$ may always be used to compare distributions because total variation cannot increase after applying a deterministic map. No Markov assumption on the sequence $\phi(Y_k)$ is needed for the bounds above.

If one instead wishes to define a Markov chain directly on the reduced representation, the original kernel must satisfy
\[
    \sum_{z:\phi(z)=b}T(y,z)
    =
    \sum_{z:\phi(z)=b}T(y',z)
\]
whenever $\phi(y)=\phi(y')$, for every representation value $b$. This is the strong lumpability condition, which makes
$\phi(Y_k)$ a Markov chain for every initial distribution
\citep{kemeny1976finite}.

\subsection{Exact sample bounds}
\label{app:direct-sample-bounds}

For $0<\rho<1$, it is sufficient to choose
\[
    q
    \ge
    \left\lceil
        \frac{\log(3/\epsilon)}
        {-\log\rho}
    \right\rceil,
    \qquad
    k=q\ell,
\]
together with
\[
    m
    \ge
    \frac{9}{4\epsilon^2}
    \left(
        \frac{2\ell}{1-\rho}+1
    \right)^2
    \left(
        \sqrt N+\sqrt{2\log(2/\delta)}
    \right)^2,
\]
and
\[
    p
    \ge
    \frac{9}{4\epsilon^2}
    \left(
        \sqrt d+\sqrt{2\log(2/\delta)}
    \right)^2.
\]
These choices make the three terms in~\eqref{eq:direct-finite-sample} at most $\epsilon/3$. When $\rho=0$, one may take $k=\ell$.

For the weighted family $\{T_\beta[O]:\beta\in I\}$ from Section~\ref{sec:convergence}, the contraction bound in Appendix~\ref{app:tv-contraction} allows $\rho=1/2$ after sufficiently many updates. Substitution gives an explicit sample requirement in terms of the uniform spectral gap $\Gamma_I[O]$ and minimum stationary mass $\pi_{I,*}$. As in the original bound, the resulting human-sample requirement has order
\[
    m
    =
    O\!\left(
        \frac{
            \log^2(2/\sqrt{\pi_{I,*}})
        }{
            \Gamma_I[O]^2
        }
        \frac{
            N+\log(1/\delta)
        }{
            \epsilon^2
        }
    \right),
\]
provided the candidate family contains a kernel that preserves $P_H$.
\end{document}

%% file: related.tex
\section{Related Work}\label{sec:related-work}

\paragraph{AI-Text Detection.}
GLTR uses token ranks to help identify generated text
\citep{gehrmann2019gltr}.
DetectGPT uses the local curvature of a language model's log
probability \citep{mitchell2023detectgpt}.
\citet{kirchenbauer2023watermark} add a detectable statistical
signal during generation.
RAID evaluates detector robustness across generators, domains,
decoding methods, and adversarial edits \citep{dugan2024raid}.

\paragraph{Detection Limits and Paraphrasing.}
\citet{sadasivan2023reliably} bound detection performance using
total variation distance.
They also evaluate recursive paraphrasing attacks.
\citet{chakraborty2024possibilities} show that sufficiently many
independent samples can distinguish distinct fixed distributions.
DIPPER rewrites text to evade detectors while largely preserving
meaning \citep{krishna2023paraphrasing}.
\citet{cheng2025adversarial} use detector feedback to train a
paraphrasing model.

\paragraph{Authorship and Style Transfer.}
\citet{grondahl2020effective} use combinatorial paraphrasing for
authorship obfuscation and imitation, with explicit evaluation
of meaning preservation.
\citet{liu2024authorship} use policy optimization for style
transfer with few samples. % from the target author.

\paragraph{Repeated Paraphrasing.}
\citet{tripto2024ship} study how repeated paraphrasing changes
authorship signals.
\citet{wang2025unveiling} observe fixed points and short cycles
under repeated paraphrasing.
\citet{geng2026markovian} model repeated rewriting as a Markov
chain, studying recurrence and sentence diversity.
\citet{zhang2024watermarks} prove limits on strong watermarking
under assumptions on evaluation and mixing of a
perturbation process.

\paragraph{Spectral Graph Theory and Markov Chains.}
\citet{chung1997spectral} develops the relation between graph
spectra, connectivity, and random walks.
\citet{jerrum1989approximating} use conductance to bound Markov
chain mixing times.
\citet{sinclair1992improved} obtains mixing bounds through
canonical paths and multicommodity flows.
\citet{diaconis1993comparison} bound the eigenvalues of a
reversible chain by comparison with another chain.
\citet{spielman2011graph} use effective resistances to construct
sparse graphs that approximately preserve Laplacian quadratic
forms.
\citet{rosenthal1995minorization} derives convergence bounds
from minorization and drift conditions.

%% file: sections/setup.tex
\section{Setup}\label{sec:setup}

We study repeated rewriting of machine-generated text toward human writing while preserving the underlying content. A content class $c$ specifies a prompt $x_c$ and the ideas that a valid response must express. Let $P_H^c$ and $P_M^c$ denote the human and initial machine distributions for that class.

Starting from $Y_0\sim P_M^c$, we repeatedly apply a rewriting kernel $T$,
\[
    Y_0\to Y_1\to\cdots\to Y_k,
    \qquad
    Y_k\sim P_M^cT^k,
\]
with the goal that the rewritten distribution approaches $P_H^c$ while every intermediate text remains a valid expression of the same content.

Human examples are used later to calibrate this rewriting process. Importantly, human and machine texts are paired at the level of content: they correspond to the same prompt and intended response.

For distributions $P,P'$ on the same countable space, let
\[
    \TV(P,P')=\frac12\sum_y |P(y)-P'(y)|.
\]
For a Markov kernel $T$, $(PT)(z)=\sum_yP(y)T(y,z)$. All logarithms are natural.

\subsection{Admissible rewrites}\label{sec:content-classes}

For each content class $c$, fix $n_c$ sentence positions and a finite set $S_{c,i}$ of allowed phrasings for sentence $i$. The admissible state space is
\begin{equation}\label{eq:2.1}
    \mathcal A_c
    =
    S_{c,1}\times\cdots\times S_{c,n_c}.
\end{equation}
Thus $y=(y_1,\ldots,y_{n_c})\in\mathcal A_c$ is one complete version of the response.

Let $Q(x,y)\in[0,1]$ measure the quality of response $y$ to prompt $x$, and let $\Sem(x,y,z)\in[0,1]$ measure how well $z$ preserves the meaning of $y$ for prompt $x$. For fixed thresholds $q_0$ and $s_0$, we require 
\begin{equation}\label{eq:2.2}
    Q(x_c,z)\ge q_0,
    \qquad
    \Sem(x_c,y,z)\ge s_0,
    \quad \forall\, y,z\in\mathcal A_c.
\end{equation}
Hence every state the rewriting process can reach remains high-quality and expresses the same intended content. The condition is imposed on complete texts: every combination in the product set~\eqref{eq:2.1} must be admissible.

We take the content classes to be pairwise disjoint and write
\[
    \mathcal Y_{\rm adm}=\bigcup_c\mathcal A_c.
\]
Sentence boundaries and order are fixed, so rewriting changes sentence phrasing but does not split, merge, or reorder sentences.

A single rewrite changes one sentence. We write $y^{i\leftarrow s}$ for the text obtained by replacing sentence $i$ of $y$ by $s\in S_{c,i}$, and $y_{-i}$ for all sentences other than $i$.

The \emph{paraphrase graph} $G_c=(\mathcal A_c,E_c)$ connects $y$ and $z$ whenever they differ in exactly one sentence:
\[
    \{y,z\}\in E_c
    \iff
    d_{\rm Ham}(y,z)=\sum_{i=1}^{n_c}\mathds{1}\{y_i\ne z_i\}=1.
\]
Thus vertices are admissible texts and edges are one-sentence rewrites. Since $\mathcal A_c$ is a full product set, any two texts can be connected by changing their differing sentences one at a time, so $G_c$ is connected. Figure~\ref{fig:paraphrase-graph} illustrates this state space for a
two-sentence example.

\begin{figure}[t]
  \centering
  \begin{tikzpicture}[
      vtx/.style={draw, circle, minimum size=10mm, inner sep=0pt, font=\small},
      change2/.style={thick},
      change1/.style={thick, dashed}]

    \node[vtx] (a1) at (0,2)   {$ab$};
    \node[vtx] (a2) at (2.4,2) {$ab'$};
    \node[vtx] (a3) at (4.8,2) {$ab''$};

    \node[vtx] (b1) at (0,0)   {$a'b$};
    \node[vtx] (b2) at (2.4,0) {$a'b'$};
    \node[vtx] (b3) at (4.8,0) {$a'b''$};

    % Change sentence 2
    \draw[change2] (a1) -- (a2);
    \draw[change2] (a2) -- (a3);
    \draw[change2] (a1) to[bend left=30] (a3);

    \draw[change2] (b1) -- (b2);
    \draw[change2] (b2) -- (b3);
    \draw[change2] (b1) to[bend right=30] (b3);

    % Change sentence 1
    \draw[change1] (a1) -- (b1);
    \draw[change1] (a2) -- (b2);
    \draw[change1] (a3) -- (b3);

  \end{tikzpicture}

  \caption{
  Paraphrase graph for a two-sentence response with
  $S_1=\{a,a'\}$ and $S_2=\{b,b',b''\}$.
  Each vertex is an admissible complete text.
  Solid edges change sentence~2; dashed edges change sentence~1.
  }
  \label{fig:paraphrase-graph}
\end{figure}

\subsection{Scoring and rewriting}\label{sec:rewriting-dynamics}

We specify which admissible rewrites are preferred. Fix a scoring language model $\pi(\cdot\mid x)$ with $\pi(y\mid x_c)>0$ for every $y\in\mathcal A_c$, and define
$
    L_c(y)=\log\pi(y\mid x_c).
$
The scoring model need not be the model that generated $P_M^c$.

For $\beta \ge 0$, define
\begin{equation}
Z_c(\beta)
=
\sum_{z\in\mathcal A_c} e^{\beta L_c(z)},
\qquad
\pi_{c,\beta}(y)
=
\frac{e^{\beta L_c(y)}}{Z_c(\beta)}.
\label{eq:target-family}
\end{equation}
The parameter $\beta$ controls how strongly the process favors high-scoring texts: $\beta=0$ gives the uniform distribution over $\mathcal A_c$, while larger $\beta$ increasingly favors texts with larger $L_c$.

Separately, let
$
    O_c:\mathcal A_c\times\mathcal A_c\to[0,1]
$
be a symmetric acceptance rule, $O_c(y,z)=O_c(z,y)$. It can encode additional restrictions on which one-sentence rewrites are allowed. We assume the graph formed by edges with $O_c(y,z)>0$ is connected. 

The distinction is: $\pi_{c,\beta}$ determines which texts are preferred, $O_c$ determines which moves are allowed.

Given the current text $y$, one rewriting step chooses a sentence position $i$ uniformly and resamples that sentence conditional on all others. Define
\[
    W_{c,\beta,i}(y_{-i})
    =
    \sum_{s\in S_{c,i}}
    e^{\beta L_c(y^{i\leftarrow s})},
\]
and
\begin{equation}\label{eq:3.6}
    p_{c,\beta,i}(s\mid y_{-i})
    =
    \frac{e^{\beta L_c(y^{i\leftarrow s})}}
    {W_{c,\beta,i}(y_{-i})}.
\end{equation}
The proposal is the random-scan Gibbs kernel
\citep{levin2017markov}
\begin{equation}\label{eq:3.7}
    K_{c,\beta}(y,z)
    =
    \frac1{n_c}
    \sum_{i=1}^{n_c}
    \mathds{1}\{y_{-i}=z_{-i}\}
    p_{c,\beta,i}(z_i\mid y_{-i}).
\end{equation}

A proposed move is accepted according to $O_c$, giving
\begin{equation}\label{eq:3.8}
T_{c,\beta}[O](y,z)
=
\begin{cases}
    O_c(y,z)K_{c,\beta}(y,z), & z\neq y,\\[1ex]
    1-\displaystyle\sum_{u\neq y}
    O_c(y,u)K_{c,\beta}(y,u), & z=y.
\end{cases}
\end{equation}

Thus one attempted update consists of choosing a sentence, sampling an admissible replacement according to the score of the resulting complete text, and accepting or rejecting that replacement. Rejected proposals count as updates but leave the text unchanged.

\begin{proposition}[Admissibility and stationarity]\label{prop:stationarity}
For every content class $c$, $T_{c,\beta}[O]$ is a Markov kernel on $\mathcal A_c$. If $O_c$ is symmetric, then it is reversible with respect to $\pi_{c,\beta}$:
\[
    \pi_{c,\beta}(y)T_{c,\beta}[O](y,z)
    =
    \pi_{c,\beta}(z)T_{c,\beta}[O](z,y).
\]
Hence $\pi_{c,\beta}$ is stationary. If the accepted graph is connected, this stationary distribution is unique. Moreover, if $Y_0\in\mathcal A_c$, then $Y_k\in\mathcal A_c$ for every $k$, so every rewrite remains admissible.
\end{proposition}

Proposition~\ref{prop:stationarity} gives the two properties: rewriting never leaves the meaning-preserving state space, and for fixed $\beta$ it targets the distribution $\pi_{c,\beta}$. The proof is given in Appendix~\ref{app:setup-proofs}.

%% file: sections/convergence_rate.tex
\section{Convergence of Repeated Rewriting}\label{sec:convergence}

We now ask how quickly repeated rewriting forgets the initial machine distribution and approaches this target.

Fix a content class $c$ and suppress the class index when there is no ambiguity. The convergence rate is governed by the \emph{spectral gap} of the rewriting kernel. Let
\[
    1=\lambda_1\ge \lambda_2\ge\cdots\ge\lambda_N\ge0
\]
denote the eigenvalues of $T_\beta[O]$ in $L^2(\pi_\beta)$, and define
\begin{equation}\label{eq:spectral-gap}
    \gamma_\beta[O]=1-\lambda_2.
\end{equation}
For a singleton class, set $\gamma_\beta[O]=1$. The nonnegativity of the eigenvalues and an equivalent variational characterization are given in Appendix~\ref{app:mixing}.

Intuitively, $\gamma_\beta[O]$ measures how easily the rewriting process can move throughout the admissible paraphrase space. A large gap corresponds to rapid mixing; a small gap indicates bottlenecks or rarely accepted transitions and therefore slow convergence.

\subsection{Geometric convergence}\label{sec:geometric-convergence}

For a fixed $\beta$, Proposition~\ref{prop:stationarity} tells us that the rewriting process targets $\pi_\beta$. We now ask how quickly it forgets the initial machine distribution.

Let $P_k=P\,T_\beta[O]^k$ denote the distribution after $k$ attempted rewrites from an initial distribution $P$. The convergence rate is controlled by the spectral gap $\gamma_\beta[O]$: larger gaps correspond to faster mixing.

The following bound follows from the spectral analysis of finite
reversible Markov chains \citep{levin2017markov}.

\begin{theorem}[Convergence to stationarity]\label{thm:mixing}
Suppose the accepted graph is connected. Then, for every $k\ge0$,
\begin{equation}\label{eq:mixing-tv}
\begin{aligned}
\TV(P_k,\pi_\beta)
&\le
\frac12\sqrt{\chi^2(P,\pi_\beta)}\,
(1-\gamma_\beta[O])^k \\
&\le
\frac12\sqrt{\pi_{\beta,\min}^{-1}-1}\,
e^{-\gamma_\beta[O]k}.
\end{aligned}
\end{equation}
where
\[
\begin{aligned}
\chi^2(P,\pi_\beta)
&=
\sum_y
\frac{(P(y)-\pi_\beta(y))^2}{\pi_\beta(y)},\quad
\pi_{\beta,\min}
&=
\min_y \pi_\beta(y).
\end{aligned}
\]
Moreover, $\TV(P_k,\pi_\beta)$ is nonincreasing in $k$ and converges to zero.
\end{theorem}

The proof is given in Appendix~\ref{app:mixing-proof}.

Thus repeated rewriting removes dependence on the initial machine distribution at an exponential rate. Up to the initial-state constant, achieving $\TV(P_k,\pi_\beta)\le\epsilon$ requires
\[
k
=
O\!\left(
\frac{1}{\gamma_\beta[O]}
\log\frac{1}{\epsilon}
\right)
\]
updates. Hence connectedness guarantees eventual convergence, while the spectral gap determines how quickly it occurs.

\subsection{Uniform convergence when $\beta$ is estimated}\label{sec:uniform-mixing}

The previous result assumes that $\beta$ is fixed. In the eventual rewriting procedure, however, $\beta$ is estimated from a finite human calibration sample. The resulting value $\widehat\beta_m$ is therefore random. A convergence guarantee at a single value of $\beta$ is not enough: we need a bound that remains valid for every value the estimator may return.

Let
$
    I=[\beta_-,\beta_+]\subset[0,\infty)
$
be the parameter range used for calibration. Define the worst-case spectral gap and minimum stationary mass over this interval by
\begin{equation}\label{eq:uniform-gap}
    \Gamma_{c,I}[O]
    =
    \inf_{\beta\in I}\gamma_{c,\beta}[O],
    \qquad
    \pi_{c,I,*}
    =
    \min_{\beta\in I,\;y\in\mathcal A_c}
    \pi_{c,\beta}(y).
\end{equation}
For a fixed finite content class and a connected accepted graph, both quantities are strictly positive. Hence, with
\begin{equation}\label{eq:uniform-mixing-bound}
    M_{c,k}[O]
    =
    \min\left\{
        1,\,
        \frac12
        \sqrt{\pi_{c,I,*}^{-1}-1}\,
        e^{-\Gamma_{c,I}[O]k}
    \right\},
\end{equation}
Theorem~\ref{thm:mixing} implies
\begin{equation}\label{eq:uniform-mixing}
    \sup_{\beta\in I}\sup_P
    \TV\!\left(
        P T_{c,\beta}[O]^k,
        \pi_{c,\beta}
    \right)
    \le
    M_{c,k}[O].
\end{equation}

This uniform bound is the quantity used later when the chain is run at the estimated parameter $\widehat\beta_m$. It separates the dynamical error due to running only finitely many rewrites from the statistical error due to estimating the human parameter. Section~\ref{sec:calibration} controls the latter, and the two errors are combined in the finite-sample convergence result.

Generic tools for lower-bounding $\Gamma_{c,I}[O]$, including conductance, canonical paths, and comparison across nearby values of $\beta$, are present in Appendix~\ref{app:mixing-tools}.

\subsection{What controls the rewriting rate?}
\label{sec:mixing-interpretation}

The spectral gap summarizes how quickly the rewriting process can explore the admissible text space. In particular, restricting which proposed rewrites are accepted can slow convergence.

The following comparison follows by comparing Dirichlet forms
\citep{diaconis1993comparison}.

\begin{proposition}[Effect of acceptance probabilities]
\label{prop:acceptance-comparison}
Suppose two rewriting rules use the same state space, scores, and proposal kernel. If
$
    O_2(y,z)\ge a\,O_1(y,z)
$
on every edge for some $a\ge0$, then
$
    \gamma_\beta[O_2]
    \ge
    a\,\gamma_\beta[O_1].
$
In particular, if every proposed rewrite is accepted with the same probability $\tau\in(0,1]$, then
$
    \gamma_\beta[O_\tau]
    =
    \tau\,\gamma_\beta[1].
$
\end{proposition}

The proof is given in Appendix~\ref{app:acceptance-proof}. Thus stricter acceptance constraints can directly slow convergence: reducing acceptance probabilities reduces the spectral gap and increases the number of rewrites required. Connectivity alone therefore does not guarantee fast mixing.

As a simple benchmark, if the score separates across $n$ sentence positions, then $\gamma_\beta[1]=1/n$; with constant acceptance probability $\tau$,
\[
    \TV\!\left(
        P T_\beta[O_\tau]^k,\pi_\beta
    \right)
    \le
    ne^{-\tau k/n}.
\]
Hence the natural mixing scale in this idealized case is
$O((n/\tau)\log(1/\epsilon))$. The result and its extension to weakly dependent sentence scores are given in Appendices~\ref{app:product-case} and~\ref{app:weak-dependence}.

We now turn from the dynamics of rewriting to the statistical question of choosing $\beta$ from human writing.

%% file: sections/drift.tex
\section{Calibrating the Human Target}\label{sec:calibration}

Section~\ref{sec:convergence} shows that, for a fixed $\beta$, repeated rewriting approaches $\pi_{c,\beta}$. We now use human examples to determine which value of $\beta$ corresponds to human writing.

\subsection{The human parameter}\label{sec:human-model}

Let
$
    I=[\beta_-,\beta_+]\subset[0,\infty)
$
be the range of values considered for $\beta$.

\begin{assumption}[Common human parameter]\label{ass:human-model}
There exists $\beta_H\in\operatorname{int}(I)$ such that, for every calibration and target class $c$,
\begin{equation}\label{eq:human-model}
    P_H^c=\pi_{c,\beta_H}.
\end{equation}
\end{assumption}

Thus the score function $L_c$ may differ across content classes, but the same value $\beta_H$ describes the human distribution within each class. This is an assumption of the model; training the scoring model on human text does not by itself imply~\eqref{eq:human-model}.

The shared value $\beta_H$ allows us to estimate the human parameter using one set of content classes and apply it to a different target class.

\subsection{Estimating $\beta_H$ from human samples}\label{sec:pseudolikelihood}

Recall the calibration sample
$
    H_m=\bigl((C_1,H_1),\ldots,(C_m,H_m)\bigr),
$
where the observations are independent, $C_j\sim\nu$, and
$H_j\mid C_j=c\sim P_H^c$.

We estimate $\beta_H$ by maximum pseudolikelihood
\citep{besag1975statistical}, averaging the conditional
log-probabilities over sentence positions and human texts.
Define
\begin{equation}\label{eq:pseudolikelihood}
    \ell_m(\beta)
    =
    \frac1m
    \sum_{j=1}^m
    \frac1{n_{C_j}}
    \sum_{i=1}^{n_{C_j}}
    \log
    p_{C_j,\beta,i}
    \bigl(H_{j,i}\mid H_{j,-i}\bigr).
\end{equation}
We define $\widehat\beta_m$ as the smallest maximizer of $\ell_m$ over $I$. Thus $\widehat\beta_m$ is the value of $\beta$ under which the observed human sentences are, on average, most likely given the rest of their text.

For the error analysis below, define
\[
    v_{c,\beta}(h)
    =
    \frac1{n_c}
    \sum_{i=1}^{n_c}
    \Var_{s\sim p_{c,\beta,i}(\cdot\mid h_{-i})}
    \left[
        L_c(h^{i\leftarrow s})
    \right].
\]
This quantity measures how much the complete-text score varies when one sentence is resampled according to the conditional distribution at parameter $\beta$.

Appendix~\ref{app:pseudolikelihood-proof} shows that $\ell_m$ is concave and gives an equivalent interpretation of $\widehat\beta_m$ in terms of the average score change produced by resampling sentences. Since the estimator uses the conditional distributions before applying $O_c$, the same calibration procedure applies to any fixed symmetric acceptance rule.

\subsection{Calibration error}\label{sec:calibration-error}

We now bound the difference between $\widehat\beta_m$ and the true human parameter $\beta_H$.

\begin{assumption}[Bounded score differences]\label{ass:bounded-scores}
There exists $0<B<\infty$ such that, for every class with positive $\nu$-probability,
\[
    |L_c(y)-L_c(z)|\le B
    \qquad
    \text{for every }\{y,z\}\in E_c.
\]
We also assume that
$
    \nu\{c:L_c\text{ is nonconstant on }\mathcal A_c\}>0.
$
\end{assumption}

The first condition bounds how much the complete-text score can change in one sentence rewrite. The second ensures that the calibration data include classes for which different admissible texts can receive different scores.

Define
\begin{equation}\label{eq:kappa}
    \kappa
    =
    \inf_{\beta\in I}
    \mathbb E\!\left[
        v_{C,\beta}(H)
    \right],
\end{equation}
where $C\sim\nu$ and $H\mid C=c\sim P_H^c$.

Under Assumptions~\ref{ass:human-model} and~\ref{ass:bounded-scores}, $\kappa>0$; see Appendix~\ref{app:positive-slope}. The quantity $\kappa$ measures how much the human conditional probabilities change as $\beta$ changes. Larger $\kappa$ makes $\beta_H$ easier to estimate.

\begin{theorem}[Calibration error]\label{thm:beta-estimation}
Under Assumptions~\ref{ass:human-model} and~\ref{ass:bounded-scores},
\[
    \widehat\beta_m\longrightarrow\beta_H
    \qquad\text{almost surely}.
\]
Moreover, with probability at least $1-\delta$,
\begin{equation}\label{eq:beta-error}
    |\widehat\beta_m-\beta_H|
    \le
    t_m(\delta)
    :=
    \frac{B}{\kappa}
    \sqrt{
        \frac{2\log(2/\delta)}{m}
    }.
\end{equation}
\end{theorem}

The proof is given in Appendix~\ref{app:beta-estimation-proof}.

Thus the error in estimating the human parameter decreases at rate
$
    |\widehat\beta_m-\beta_H|
    =
    O_p(m^{-1/2}).
$

Theorem~\ref{thm:beta-estimation} shows that the human target can be learned from finitely many human examples: the estimation error shrinks as $m^{-1/2}$, so additional calibration data directly improves how accurately we recover $\beta_H$.

\subsection{From parameter error to distributional error}
\label{sec:parameter-sensitivity}

The final step is to translate the error in $\widehat\beta_m$ into a difference between text distributions. Define
\[
    \sigma_c
    =
    \sup_{\beta\in I}
    \sqrt{
        \Var_{\pi_{c,\beta}}(L_c)
    }.
\]

\begin{lemma}[Sensitivity to $\beta$]\label{lem:parameter-sensitivity}
For any $a,b\in I$,
\begin{equation}\label{eq:tv-sensitivity}
    \TV(\pi_{c,a},\pi_{c,b})
    \le
    \frac{\sigma_c}{2}|a-b|.
\end{equation}
\end{lemma}

The proof is given in Appendix~\ref{app:sensitivity-proof}.

Thus, if $\widehat\beta_m$ is close to $\beta_H$, then the corresponding stationary distributions are also close:
\[
    \widehat\beta_m\approx\beta_H
    \quad\Longrightarrow\quad
    \pi_{c,\widehat\beta_m}
    \approx
    \pi_{c,\beta_H}
    =
    P_H^c.
\]

Lemma~\ref{lem:parameter-sensitivity} is the bridge from parameter estimation to our actual objective: it shows that accurately estimating $\beta_H$ also guarantees that the stationary distribution targeted by the rewriting process is close to the true human distribution.

%% file: sections/sample_complexity.tex
\section{Finite-Sample Convergence to Human Writing}
\label{sec:finite-sample-convergence}

We now combine the two sources of error developed in the previous sections. After finitely many rewrites, the chain may not yet have reached its stationary distribution, and with finitely many human examples, $\widehat\beta_m$ may differ from $\beta_H$.

Fix a target content class $c$ and suppress the class index throughout this section. Write $P_M$ and $P_H$ for its machine and human distributions, $\sigma=\sigma_c$, and $M_k[O]=M_{c,k}[O]$.

\subsection{Combining rewriting and calibration}\label{sec:finite-convergence}

Conditional on the human calibration sample, define
\[
    Q_{m,k}
    =
    P_M T_{\widehat\beta_m}[O]^k.
\]
Thus $Q_{m,k}$ is the distribution of a machine response after $k$ rewrites using the parameter estimated from $m$ human calibration texts.

\begin{theorem}[Finite-sample convergence]\label{thm:finite-sample}
Under the assumptions of Sections~\ref{sec:convergence} and~\ref{sec:calibration}, with probability at least $1-\delta$ over the calibration sample,
\begin{equation}\label{eq:finite-sample}
\TV(Q_{m,k},P_H)
\le
M_k[O]
+
\frac{\sigma B}{2\kappa}
\sqrt{\frac{2\log(2/\delta)}{m}}.
\end{equation}
The first term is the error from running only finitely many rewrites, while
the second is the error from estimating the human parameter from finitely
many calibration texts.
\end{theorem}

The proof is given in Appendix~\ref{app:finite-sample-proof}.

Theorem~\ref{thm:finite-sample} is the main consequence of the preceding analysis: more rewriting reduces the first term, while more human calibration data reduces the second. In particular, if $k_m\to\infty$, then
\[
    \TV(Q_{m,k_m},P_H)\longrightarrow0
    \qquad\text{almost surely}.
\]

The target class need not appear in the calibration sample; the result relies only on the shared human parameter from Assumption~\ref{ass:human-model}.

For fixed $m$, rewriting approaches $\pi_{\widehat\beta_m}$ rather than necessarily $P_H$. Hence the distance to $P_H$ need not decrease after every individual rewrite when $\widehat\beta_m\neq\beta_H$.

\subsection{Empirical rewritten outputs}\label{sec:empirical-outputs}

Theorem~\ref{thm:finite-sample} concerns the distribution of a rewritten response. In practice, we observe only finitely many rewritten outputs.

Generate $p$ independent rewritten outputs $Y_1^{(k)},\ldots,Y_p^{(k)}$, each using the same estimated parameter $\widehat\beta_m$, and define their empirical distribution by
\[
    \widehat Q_{m,k,p}
    =
    \frac1p\sum_{i=1}^p\delta_{Y_i^{(k)}}.
\]
The initial machine texts may be fixed or drawn from different initial machine distributions.

To compare outputs through a finite representation, let
\[
    \phi:\mathcal A_c\to\{1,\ldots,d\}.
\]
For any distribution $P$, write $P^\phi$ for the distribution of $\phi(Y)$ when $Y\sim P$. Taking $\phi$ to be the identity gives the full-text distribution, with $d=|\mathcal A_c|$.

\begin{theorem}[Empirical output distribution]\label{thm:empirical-output}
Under the assumptions of Theorem~\ref{thm:finite-sample}, with probability at least $1-\delta$ over calibration and rewriting,
\begin{align}
\TV\!\left(
    \widehat Q_{m,k,p}^{\,\phi},
    P_H^\phi
\right)
&\le
M_k[O]
+
\frac{\sigma B}{2\kappa}
\sqrt{\frac{2\log(4/\delta)}{m}}
\nonumber\\
&\qquad+
\frac{
    \sqrt d+\sqrt{2\log(2/\delta)}
}{
    2\sqrt p
}.
\label{eq:empirical-output}
\end{align}
The three terms correspond respectively to finite rewriting, calibration
from finitely many human examples, and estimating the rewritten
distribution from finitely many outputs.
\end{theorem}

The proof is given in Appendix~\ref{app:empirical-output-proof}.

Using a representation with $d\ll|\mathcal A_c|$ reduces the number of rewritten outputs needed to estimate the represented distribution, although the guarantee then applies only after applying $\phi$.

\subsection{Sample complexity}\label{sec:sample-complexity}

Write
$
    \Gamma=\Gamma_{c,I}[O].
$
Up to the constants defined above, Theorem~\ref{thm:finite-sample} shows that total-variation error at most $\epsilon$ can be obtained with
\[
    k
    =
    O\!\left(
        \frac1{\Gamma}\log\frac1\epsilon
    \right),
    \qquad
    m
    =
    O\!\left(
        \frac{\sigma^2B^2}
        {\kappa^2\epsilon^2}
        \log\frac1\delta
    \right).
\]
For the empirical distribution in Theorem~\ref{thm:empirical-output}, it additionally suffices to take
\[
    p
    =
    O\!\left(
        \frac{d+\log(1/\delta)}
        {\epsilon^2}
    \right).
\]

The number of rewriting steps therefore depends inversely on the mixing rate $\Gamma$ and only logarithmically on the desired accuracy. The number of human calibration texts and observed rewritten outputs have the usual $\epsilon^{-2}$ sampling dependence.

A benefit of the shared-parameter model is that calibration estimates only the scalar $\beta_H$: the bound on $m$ has no explicit dependence on the number of possible texts in the target class. The constants $B$, $\sigma$, and $\kappa$ may nevertheless depend on text length, and the full state space $|\mathcal A_c|=\prod_i|S_{c,i}|$ can grow rapidly with the number of sentences.

Exact sufficient choices of $k$, $m$, and $p$, together with computational details, are given in Appendix~\ref{app:sample-bounds} and Appendix~\ref{app:computation}.

%% file: sections/convergence_without_tempered.tex
\section{Convergence Without the Shared Human Parameter}
\label{sec:direct-kernel}

Sections~\ref{sec:calibration}--\ref{sec:finite-sample-convergence} assume that human writing satisfies
$P_H^c=\pi_{c,\beta_H}$ for a shared parameter $\beta_H$. We now remove this assumption. Instead, we learn a rewriting kernel directly from human examples and ask when repeated rewriting approaches the human distribution.

Fix a target content class $c$ and suppress the class index. Write $\mathcal A=\mathcal A_c$, $N=|\mathcal A|$, and let $P_H$ and $P_M$ denote the human and initial machine distributions on this class.

Unlike the shared-parameter setting, we now use human examples from the target class itself. Let $H_1,\ldots,H_m\sim P_H$ 
be independent human texts and define their empirical distribution by$  \widehat P_{H,m}
    =
    \frac1m\sum_{j=1}^m\delta_{H_j}.$ 

\subsection{Contraction and preservation of human writing}
\label{sec:direct-contraction}

Let $\widehat T_m$ be a rewriting kernel on $\mathcal A$ learned from the human sample. We require two properties.

First, repeated application of the kernel should reduce differences between starting distributions. Specifically, suppose that for some integer $\ell\ge1$ and $0\le\rho<1$,
\begin{equation}\label{eq:contraction}
    \TV(P\widehat T_m^\ell,Q\widehat T_m^\ell)
    \le
    \rho\,\TV(P,Q)
\end{equation}
for all distributions $P,Q$ on $\mathcal A$. Thus every block of $\ell$ rewrites reduces total-variation distance by at least a factor $\rho$. Such contraction conditions are standard tools for controlling the
mixing of local-update Markov chains
\citep{dobrushin1970prescribing,hayes2006simple}.

Second, the kernel should have little effect on human writing. Define
\begin{equation}\label{eq:stationarity-error}
    \epsilon_m
    =
    \TV(
        \widehat P_{H,m}\widehat T_m,
        \widehat P_{H,m}
    ).
\end{equation}
The quantity $\epsilon_m$ measures how much one rewrite changes the empirical human distribution.

These two properties are enough to control repeated rewriting.

\begin{theorem}[Convergence with human-preservation error]
\label{thm:approx-preservation}
Let $k=q\ell$ for an integer $q\ge0$. If~\eqref{eq:contraction} holds, then
\begin{equation}\label{eq:approx-preservation}
\TV(
    P_M\widehat T_m^k,
    \widehat P_{H,m}
)
\le
\rho^q
\TV(P_M,\widehat P_{H,m})
+
\ell\frac{1-\rho^q}{1-\rho}\epsilon_m.
\end{equation}
In particular,
\[
\TV(
    P_M\widehat T_m^k,
    \widehat P_{H,m}
)
\le
\rho^q
+
\frac{\ell}{1-\rho}\epsilon_m.
\]
\end{theorem}

The proof and the corresponding result for arbitrary $k$ are given in Appendix~\ref{app:direct-general-convergence}.

The two terms have distinct meanings. The first is the remaining dependence on the initial machine distribution and decreases geometrically with rewriting. The second is the error caused by the fact that the learned kernel may not preserve human writing exactly. If $\epsilon_m=0$, repeated rewriting converges geometrically to the empirical human distribution.

\subsection{Learning a human-preserving kernel}
\label{sec:learning-kernel}

We now choose $\widehat T_m$ using the human examples. Let $\mathcal T$ be a finite collection of admissible rewriting kernels on $\mathcal A$, chosen independently of the human sample, and suppose every $T\in\mathcal T$ satisfies~\eqref{eq:contraction} with the same $\ell$ and $\rho$.

We choose the kernel that changes the empirical human distribution as little as possible:
\begin{equation}\label{eq:kernel-estimator}
    \widehat T_m
    \in
    \arg\min_{T\in\mathcal T}
    \TV(
        \widehat P_{H,m}T,
        \widehat P_{H,m}
    ),
\end{equation}
using a fixed rule to resolve ties.

Assume that the candidate set contains at least one kernel $T^\star$ that preserves the true human distribution,
\begin{equation}\label{eq:human-preserving-kernel}
    P_HT^\star=P_H.
\end{equation}

\begin{theorem}[Human-preservation error]
\label{thm:human-preservation}
Under~\eqref{eq:human-preserving-kernel}, with probability at least $1-\delta$,
\begin{equation}\label{eq:human-preservation-error}
    \epsilon_m
    \le
    2\TV(\widehat P_{H,m},P_H)
    \le
    \frac{
        \sqrt N+\sqrt{2\log(1/\delta)}
    }{
        \sqrt m
    }.
\end{equation}
\end{theorem}

The proof is given in Appendix~\ref{app:human-preservation-proof}.

Thus the learned rewriting kernel becomes increasingly stable on human writing as more human examples are observed, with error decreasing at rate $m^{-1/2}$. Combined with Theorem~\ref{thm:approx-preservation}, this implies that repeated rewriting can approach human writing without assuming the shared-parameter model from Section~\ref{sec:human-model}.

\subsection{Finite-sample guarantee}
\label{sec:direct-finite-sample}

As in Section~\ref{sec:empirical-outputs}, suppose we observe $p$ independently rewritten outputs
$
    Y_1^{(k)},\ldots,Y_p^{(k)}
    \sim
    P_M\widehat T_m^k
$
conditional on the human sample, and define
\[
    \widehat P_{M,k,p}
    =
    \frac1p\sum_{i=1}^p\delta_{Y_i^{(k)}}.
\]

Let $\phi:\mathcal A\to\{1,\ldots,d\}$ be the representation used for comparison, and write $P^\phi$ for the distribution after applying $\phi$.

\begin{theorem}[Finite-sample convergence]
\label{thm:direct-finite-sample}
Suppose $k=q\ell$ and the assumptions of Theorem~\ref{thm:human-preservation} hold. Then, with probability at least $1-\delta$,
\begin{align}
\TV\!\left(
    \widehat P_{M,k,p}^{\,\phi},
    P_H^\phi
\right)
&\le
\quad
\rho^q
\nonumber\\
&\quad+
\left(
    \frac{2\ell}{1-\rho}+1
\right)
\frac{
    \sqrt N+\sqrt{2\log(2/\delta)}
}{
    2\sqrt m
}
\nonumber\\
&\quad+
\frac{
    \sqrt d+\sqrt{2\log(2/\delta)}
}{
    2\sqrt p
}.
\label{eq:direct-finite-sample}
\end{align}

Here the first term measures the remaining dependence on the initial
machine distribution, the second comes from learning the rewriting kernel
from finitely many human examples, and the third is the sampling error from
observing only $p$ rewritten outputs.
\end{theorem}

The proof is given in Appendix~\ref{app:direct-finite-sample-proof}.

The bound again separates the sources of error: repeated rewriting removes dependence on the initial machine text, human samples determine how well the learned kernel preserves human writing, and finitely many rewritten outputs introduce sampling error.

For $0<\rho<1$, the corresponding sample requirements have order
\[
    k
    =
    O\!\left(
        \ell
        \frac{\log(1/\epsilon)}{-\log\rho}
    \right),
\]
\[
\begin{aligned}
m
&=
O\!\left(
    \left(
        1+\frac{\ell}{1-\rho}
    \right)^2
    \frac{N+\log(1/\delta)}{\epsilon^2}
\right), \\
p
&=
O\!\left(
    \frac{d+\log(1/\delta)}{\epsilon^2}
\right).
\end{aligned}
\]
If $\rho=0$, one block of $\ell$ rewrites removes the initial-distribution term.

The tradeoff with Sections~\ref{sec:calibration}--\ref{sec:finite-sample-convergence} is now explicit. The shared-parameter approach estimates only the scalar $\beta_H$ and therefore avoids an explicit dependence on the number $N$ of admissible texts. The approach in this section makes no shared-parameter assumption, but learning a kernel that preserves the human distribution requires enough human samples to estimate that distribution over the state space.

\section{Discussion and Limitations}

Our results characterize when repeated rewriting can reduce the distributional gap between machine and human writing. In the shared-parameter setting, human writing is modeled by a common inverse-temperature parameter $\beta_H$: rewriting at the estimated temperature $\widehat\beta_m$ drives machine text toward the corresponding stationary distribution, with the overall error separating into finite rewriting error and finite calibration error. As more rewrites are applied, the effect of the initial machine distribution decays according to the mixing rate; as more human examples are observed, $\widehat\beta_m$ concentrates around $\beta_H$. Without the shared-temperature assumption, a similar convergence is obtained by learning a rewriting kernel that both contracts differences between distributions and approximately preserves human writing.

The guarantees depend on quantities determined by the particular rewrite space and scoring model, such as the spectral gap, $\kappa$, and the size of the admissible space. These can deteriorate when rewrite constraints are strict, the state space contains bottlenecks, or texts become longer.

%% file: references.bib
@inproceedings{gehrmann2019gltr,
  title     = {{GLTR}: Statistical Detection and Visualization of Generated Text},
  author    = {Gehrmann, Sebastian and Strobelt, Hendrik and Rush, Alexander},
  booktitle = {Proceedings of the 57th Annual Meeting of the Association for Computational Linguistics: System Demonstrations},
  pages     = {111--116},
  year      = {2019},
  publisher = {Association for Computational Linguistics},
  address   = {Florence, Italy},
  doi       = {10.18653/v1/P19-3019},
  url       = {https://aclanthology.org/P19-3019/}
}

@inproceedings{mitchell2023detectgpt,
  title     = {{DetectGPT}: Zero-Shot Machine-Generated Text Detection using Probability Curvature},
  author    = {Mitchell, Eric and Lee, Yoonho and Khazatsky, Alexander and Manning, Christopher D. and Finn, Chelsea},
  booktitle = {Proceedings of the 40th International Conference on Machine Learning},
  series    = {Proceedings of Machine Learning Research},
  volume    = {202},
  pages     = {24950--24962},
  year      = {2023},
  publisher = {PMLR},
  url       = {https://proceedings.mlr.press/v202/mitchell23a.html}
}

@inproceedings{kirchenbauer2023watermark,
  title     = {A Watermark for Large Language Models},
  author    = {Kirchenbauer, John and Geiping, Jonas and Wen, Yuxin and Katz, Jonathan and Miers, Ian and Goldstein, Tom},
  booktitle = {Proceedings of the 40th International Conference on Machine Learning},
  series    = {Proceedings of Machine Learning Research},
  volume    = {202},
  pages     = {17061--17084},
  year      = {2023},
  publisher = {PMLR},
  url       = {https://proceedings.mlr.press/v202/kirchenbauer23a.html}
}

@inproceedings{dugan2024raid,
  title     = {{RAID}: A Shared Benchmark for Robust Evaluation of Machine-Generated Text Detectors},
  author    = {Dugan, Liam and Hwang, Alyssa and Trhl{\'i}k, Filip and Zhu, Andrew and Ludan, Josh Magnus and Xu, Hainiu and Ippolito, Daphne and Callison-Burch, Chris},
  booktitle = {Proceedings of the 62nd Annual Meeting of the Association for Computational Linguistics (Volume 1: Long Papers)},
  pages     = {12463--12492},
  year      = {2024},
  publisher = {Association for Computational Linguistics},
  address   = {Bangkok, Thailand},
  doi       = {10.18653/v1/2024.acl-long.674},
  url       = {https://aclanthology.org/2024.acl-long.674/}
}

@article{sadasivan2023reliably,
  title   = {Can {AI}-Generated Text be Reliably Detected?},
  author  = {Sadasivan, Vinu Sankar and Kumar, Aounon and Balasubramanian, Sriram and Wang, Wenxiao and Feizi, Soheil},
  journal = {arXiv preprint arXiv:2303.11156},
  year    = {2023},
  url     = {https://arxiv.org/abs/2303.11156}
}

@inproceedings{chakraborty2024possibilities,
  title     = {Position: On the Possibilities of {AI}-Generated Text Detection},
  author    = {Chakraborty, Souradip and Bedi, Amrit and Zhu, Sicheng and An, Bang and Manocha, Dinesh and Huang, Furong},
  booktitle = {Proceedings of the 41st International Conference on Machine Learning},
  series    = {Proceedings of Machine Learning Research},
  volume    = {235},
  pages     = {6093--6115},
  year      = {2024},
  publisher = {PMLR},
  url       = {https://proceedings.mlr.press/v235/chakraborty24a.html}
}

@inproceedings{krishna2023paraphrasing,
  title     = {Paraphrasing Evades Detectors of {AI}-Generated Text, but Retrieval Is an Effective Defense},
  author    = {Krishna, Kalpesh and Song, Yixiao and Karpinska, Marzena and Wieting, John and Iyyer, Mohit},
  booktitle = {Advances in Neural Information Processing Systems},
  volume    = {36},
  year      = {2023},
  doi       = {10.52202/075280-1195},
  url       = {https://proceedings.neurips.cc/paper_files/paper/2023/hash/575c450013d0e99e4b0ecf82bd1afaa4-Abstract-Conference.html}
}

@inproceedings{cheng2025adversarial,
  title     = {Adversarial Paraphrasing: A Universal Attack for Humanizing {AI}-Generated Text},
  author    = {Cheng, Yize and Sadasivan, Vinu Sankar and Saberi, Mehrdad and Saha, Shoumik and Feizi, Soheil},
  booktitle = {Advances in Neural Information Processing Systems},
  volume    = {38},
  year      = {2025},
  url       = {https://papers.nips.cc/paper_files/paper/2025/hash/443f314cd420ce621b6e748fd1194ed8-Abstract-Conference.html}
}

@article{grondahl2020effective,
  title   = {Effective Writing Style Transfer via Combinatorial Paraphrasing},
  author  = {Gr{\"o}ndahl, Tommi and Asokan, N.},
  journal = {Proceedings on Privacy Enhancing Technologies},
  volume  = {2020},
  number  = {4},
  pages   = {175--195},
  year    = {2020},
  doi     = {10.2478/popets-2020-0068},
  url     = {https://www.petsymposium.org/popets/2020/popets-2020-0068.php}
}

@article{liu2024authorship,
  title   = {Authorship Style Transfer with Policy Optimization},
  author  = {Liu, Shuai and Agarwal, Shantanu and May, Jonathan},
  journal = {arXiv preprint arXiv:2403.08043},
  year    = {2024},
  url     = {https://arxiv.org/abs/2403.08043}
}

@inproceedings{tripto2024ship,
  title     = {A Ship of Theseus: Curious Cases of Paraphrasing in {LLM}-Generated Texts},
  author    = {Tripto, Nafis Irtiza and Venkatraman, Saranya and Macko, Dominik and Moro, Robert and Srba, Ivan and Uchendu, Adaku and Le, Thai and Lee, Dongwon},
  booktitle = {Proceedings of the 62nd Annual Meeting of the Association for Computational Linguistics (Volume 1: Long Papers)},
  pages     = {6608--6625},
  year      = {2024},
  publisher = {Association for Computational Linguistics},
  address   = {Bangkok, Thailand},
  doi       = {10.18653/v1/2024.acl-long.357},
  url       = {https://aclanthology.org/2024.acl-long.357/}
}

@inproceedings{wang2025unveiling,
  title     = {Unveiling Attractor Cycles in Large Language Models: A Dynamical Systems View of Successive Paraphrasing},
  author    = {Wang, Zhilin and Li, Yafu and Yan, Jianhao and Cheng, Yu and Zhang, Yue},
  booktitle = {Proceedings of the 63rd Annual Meeting of the Association for Computational Linguistics (Volume 1: Long Papers)},
  pages     = {12740--12755},
  year      = {2025},
  publisher = {Association for Computational Linguistics},
  address   = {Vienna, Austria},
  doi       = {10.18653/v1/2025.acl-long.624},
  url       = {https://aclanthology.org/2025.acl-long.624/}
}

@article{geng2026markovian,
  title   = {Markovian Generation Chains in Large Language Models},
  author  = {Geng, Mingmeng and Mohamed, Amr and Shang, Guokan and Vazirgiannis, Michalis and Poibeau, Thierry},
  journal = {arXiv preprint arXiv:2603.11228},
  year    = {2026},
  url     = {https://arxiv.org/abs/2603.11228}
}

@inproceedings{zhang2024watermarks,
  title     = {Watermarks in the Sand: Impossibility of Strong Watermarking for Language Models},
  author    = {Zhang, Hanlin and Edelman, Benjamin L. and Francati, Danilo and Venturi, Daniele and Ateniese, Giuseppe and Barak, Boaz},
  booktitle = {Proceedings of the 41st International Conference on Machine Learning},
  series    = {Proceedings of Machine Learning Research},
  volume    = {235},
  pages     = {58851--58880},
  year      = {2024},
  publisher = {PMLR},
  url       = {https://proceedings.mlr.press/v235/zhang24o.html}
}

@article{jerrum1989approximating,
  title   = {Approximating the Permanent},
  author  = {Jerrum, Mark and Sinclair, Alistair},
  journal = {SIAM Journal on Computing},
  volume  = {18},
  number  = {6},
  pages   = {1149--1178},
  year    = {1989},
  doi     = {10.1137/0218077}
}

@article{sinclair1992improved,
  title   = {Improved Bounds for Mixing Rates of Markov Chains and Multicommodity Flow},
  author  = {Sinclair, Alistair},
  journal = {Combinatorics, Probability and Computing},
  volume  = {1},
  number  = {4},
  pages   = {351--370},
  year    = {1992},
  doi     = {10.1017/S0963548300000390}
}

@article{diaconis1993comparison,
  title   = {Comparison Theorems for Reversible Markov Chains},
  author  = {Diaconis, Persi and Saloff-Coste, Laurent},
  journal = {The Annals of Applied Probability},
  volume  = {3},
  number  = {3},
  pages   = {696--730},
  year    = {1993},
  doi     = {10.1214/aoap/1177005359}
}

@book{levin2017markov,
  title     = {Markov Chains and Mixing Times},
  author    = {Levin, David A. and Peres, Yuval},
  edition   = {2},
  publisher = {American Mathematical Society},
  address   = {Providence, Rhode Island},
  year      = {2017},
  note      = {With contributions by Elizabeth L. Wilmer}
}

@article{dobrushin1970prescribing,
  title   = {Prescribing a System of Random Variables by Conditional Distributions},
  author  = {Dobrushin, R. L.},
  journal = {Theory of Probability and Its Applications},
  volume  = {15},
  number  = {3},
  pages   = {458--486},
  year    = {1970},
  doi     = {10.1137/1115049}
}

@inproceedings{hayes2006simple,
  title     = {A Simple Condition Implying Rapid Mixing of Single-Site Dynamics on Spin Systems},
  author    = {Hayes, Thomas P.},
  booktitle = {47th Annual IEEE Symposium on Foundations of Computer Science},
  pages     = {39--46},
  year      = {2006},
  publisher = {IEEE},
  doi       = {10.1109/FOCS.2006.6}
}

@article{dyer2009matrix,
  title   = {Matrix Norms and Rapid Mixing for Spin Systems},
  author  = {Dyer, Martin and Goldberg, Leslie Ann and Jerrum, Mark},
  journal = {The Annals of Applied Probability},
  volume  = {19},
  number  = {1},
  pages   = {71--107},
  year    = {2009},
  doi     = {10.1214/08-AAP532}
}

@article{besag1975statistical,
  title   = {Statistical Analysis of Non-Lattice Data},
  author  = {Besag, Julian},
  journal = {The Statistician},
  volume  = {24},
  number  = {3},
  pages   = {179--195},
  year    = {1975},
  doi     = {10.2307/2987782}
}

@book{chung1997spectral,
  title = {Spectral Graph Theory},
  author = {Chung, Fan R. K.},
  series = {CBMS Regional Conference Series in Mathematics},
  volume = {92},
  publisher = {American Mathematical Society},
  year = {1997},
  url = {https://fanchung.ucsd.edu/research/revised.html}
}

@article{spielman2011graph,
  title = {Graph Sparsification by Effective Resistances},
  author = {Spielman, Daniel A. and Srivastava, Nikhil},
  journal = {SIAM Journal on Computing},
  volume = {40},
  number = {6},
  pages = {1913--1926},
  year = {2011},
  doi = {10.1137/080734029},
  url = {https://epubs.siam.org/doi/10.1137/080734029}
}

@article{rosenthal1995minorization,
  title = {Minorization Conditions and Convergence Rates for {Markov} Chain {Monte Carlo}},
  author = {Rosenthal, Jeffrey S.},
  journal = {Journal of the American Statistical Association},
  volume = {90},
  number = {430},
  pages = {558--566},
  year = {1995},
  doi = {10.1080/01621459.1995.10476548},
  url = {https://probability.ca/jeff/ftpdir/minor.pdf}
}

@inproceedings{bubley1997path,
  author = {Bubley, Russ and Dyer, Martin},
  title = {Path Coupling: A Technique for Proving Rapid Mixing in {Markov} Chains},
  booktitle = {Proceedings of the 38th Annual Symposium on Foundations of Computer Science},
  pages = {223--231},
  publisher = {IEEE},
  year = {1997},
  doi = {10.1109/SFCS.1997.646111},
  url = {https://www.math.cmu.edu/~af1p/Teaching/Markov_Chain_Mixing/Papers/Pathcouple.pdf}
}

@article{wainwright2008graphical,
  author = {Wainwright, Martin J. and Jordan, Michael I.},
  title = {Graphical Models, Exponential Families, and Variational Inference},
  journal = {Foundations and Trends in Machine Learning},
  volume = {1},
  number = {1--2},
  pages = {1--305},
  year = {2008},
  doi = {10.1561/2200000001},
  url = {https://www.nowpublishers.com/article/Details/MAL-001}
}

@article{hoeffding1963probability,
  author = {Hoeffding, Wassily},
  title = {Probability Inequalities for Sums of Bounded Random Variables},
  journal = {Journal of the American Statistical Association},
  volume = {58},
  number = {301},
  pages = {13--30},
  year = {1963},
  doi = {10.1080/01621459.1963.10500830},
  url = {https://www.cs.rpi.edu/academics/courses/spring06/random/hoefding.pdf}
}

@incollection{mcdiarmid1989bounded,
  author = {McDiarmid, Colin},
  title = {On the Method of Bounded Differences},
  booktitle = {Surveys in Combinatorics, 1989},
  editor = {Siemons, J.},
  series = {London Mathematical Society Lecture Note Series},
  volume = {141},
  pages = {148--188},
  publisher = {Cambridge University Press},
  year = {1989},
  doi = {10.1017/CBO9781107359949.008}
}

@article{seneta1979coefficients,
  author = {Seneta, E.},
  title = {Coefficients of Ergodicity: Structure and Applications},
  journal = {Advances in Applied Probability},
  volume = {11},
  number = {3},
  pages = {576--590},
  year = {1979},
  doi = {10.2307/1426955}
}

@book{kemeny1976finite,
  author = {Kemeny, John G. and Snell, J. Laurie},
  title = {Finite {Markov} Chains},
  series = {Undergraduate Texts in Mathematics},
  publisher = {Springer-Verlag},
  year = {1976},
  note = {Reprint of the 1960 edition with a new appendix},
  url = {https://link.springer.com/book/9780387901923}
}
